\documentclass[12pt]{article}
\usepackage{geometry}
\usepackage{amsmath}
\usepackage{amsfonts}
\usepackage{amsthm} 

\RequirePackage{tgtermes}
\RequirePackage{newtxtext}
\RequirePackage{newtxmath}
\RequirePackage{bm}
\RequirePackage{endnotes}
\usepackage{dsfont}
\usepackage{enumitem}
\usepackage{array}
\usepackage{pifont}
\usepackage{tikz}
\usepackage{stackengine}
\usepackage{caption}
\usepackage{wrapfig}
\usepackage{subcaption}
\usepackage[normalem]{ulem}
\usetikzlibrary{arrows.meta, positioning}

\usepackage[english]{babel}

\newcommand{\EE}{\mathbb E}

\usepackage{algorithm}
\usepackage{algpseudocode}

\usepackage{natbib}
\bibpunct[, ]{(}{)}{,}{a}{}{,}

\usepackage{fancyhdr}
\usepackage{hyperref}
\providecommand{\Description}[1]{}

\newtheorem{theorem}{Theorem}
\newtheorem{lemma}[theorem]{Lemma}
\newtheorem{proposition}[theorem]{Proposition}
\newtheorem{corollary}[theorem]{Corollary}
\newtheorem{claim}[theorem]{Claim}
\theoremstyle{definition}

\theoremstyle{remark}

\begin{document}

\title{Right-Sizing Communication and Recommendation Set Size in AI-Assisted Search}

\author{
\small Jing Dong\\
\small Columbia Business School, Columbia University\\
\small \texttt{jing.dong@gsb.columbia.edu}
\and
\small Prakirt Raj Jhunjhunwala\\
\small Amazon.com Inc.\\
\small \texttt{prakirt2203@gmail.com}
\and
\small Yash Kanoria\\
\small Columbia Business School, Columbia University\\
\small \texttt{ykanoria@gmail.com}
}

\date{}

\maketitle

\begin{abstract}
We model the interaction between a user and an AI-driven recommendation system. The user initiates the process by conveying preference information through a costly and noisy message. The AI assistant, acting as a Bayesian agent, interprets the user's message to form a posterior belief about their true preferences and make product recommendations. In particular, it determines how many recommendations to present so as to maximize the user's expected utility from their final choice, while accounting for the search cost induced by the size of the recommendation set. We use mutual information based cost functions to model the two distinct costs incurred by the user during the interaction: (i) a \textit{communication cost}, which increases with the precision of their preference message, and (ii) a \textit{search cost}, which increases with the size of the recommendation set provided by the AI assistant.

We study products and preferences which live in $d$ dimensional space, and ask how the user's expected payoff can be maximized. For large $d$, we characterize how optimal message precision and recommendation set size depend on the cost parameters, under two distinct distributions from which recommendations can be sampled from the product universe: (i) \textit{Bayes\' posterior belief}, and (ii) an optimized \textit{tilted distribution}. Under the posterior sampling scheme (i), we identify a \textit{hybrid regime}, in which an efficient interaction policy requires jointly optimizing the amount of information (in bits) conveyed by the user and the number of recommendations provided by the AI assistant. In the tilted sampling scheme (ii), our results show that the optimal interaction policy uses only one of communication and search, favoring whichever of them is less costly.
\end{abstract}


\bigskip

\section{Introduction}
AI-powered recommendation systems are increasingly embedded in customer-facing platforms across domains such as e-commerce \citep{bansal2025magentic, cui2017superagent}, healthcare \citep{moor2023foundation, jiang2017artificial, yu2018artificial}, and education \citep{kasneci2023chatgpt, atchley2024human}. As these systems proliferate, the nature of human decision-making is shifting from direct search to interactive delegation. This shift is particularly salient in e-commerce environments, where customers routinely face thousands of options for a single product category—for example, a search for headphones on Amazon can return an overwhelming number of results. Carefully reading product specifications and reviews to identify the option that best matches one's preferences can therefore be extremely time-consuming and cognitively demanding \citep{schwartz2015paradox, scheibehenne2010can}. AI shopping assistants can alleviate this burden by aggregating product information and generating recommendations based on a posterior belief over the customer's latent preferences. In this work, we use the term \textit{agent} to refer to an AI shopping assistant whose objective is to optimally assist the customer. (In particular, our shopping ``agents'' will neither be strategic nor enjoy much autonomy.) We also refer to customers as \textit{users}.

In modern recommendation settings, agents excel at processing large, complex product spaces and narrowing down feasible options, but they typically lack direct access to users’ true preferences. Eliciting this information requires communication from the user, an inherently costly activity for the user due to the cognitive effort and time spent. As a result, agents rely on limited user interactions to infer preferences and tailor recommendations. Our work studies how to optimally structure this user–agent interaction in a product recommendation setting. Here, \textit{interaction} refers specifically to the exchange of information between the user and the agent, namely the user’s communication of preference information and the agent’s provision of product recommendations. The key mechanism governing this interaction is that user preference communication improves the relevance of the recommendation set, while the agent’s provision of multiple recommendations addresses residual uncertainty in those preferences.

\smallskip
\noindent\textbf{Research Question.} \emph{How should one jointly choose the amount of preference information communicated by the user and the number of recommendations provided to the user, to balance communication costs, search costs, and product utility?}
\smallskip

This paper introduces a theoretical framework to study AI-assisted decision-making, where the user obtains utility from the selected item and incurs a \textit{communication cost} for specifying preferences and a \textit{search cost} for evaluating the recommendation set. Upon receiving a user message, the agent constructs a sampling distribution and generates a recommendation menu consisting of independent samples from this distribution. The menu size is optimized to balance two competing objectives: maximizing the expected utility of the best option in the set (i.e., ensuring the menu contains a highly preferred item), while keeping the set small enough to limit the user's search costs. Anticipating the agent’s response, the user can carefully choose the precision of their message, trading off the benefits of a more targeted recommendation set against the cost of communication. Our analysis uses a high-dimensional approximation in terms of the number of features and reveals a sharp distinction between two practically motivated recommendation sampling schemes. The default approach (for a Gen AI-based agent) is posterior sampling, whereby each recommendation is an independent and identically distributed (i.i.d.) draw from the posterior distribution over the user’s preferred product. Under posterior sampling, optimal performance generally requires an optimal blend of user preference communication and multiple product recommendations. In contrast, under optimally designed importance sampling of product recommendations (which we call sampling from a ``tilted'' distribution), the optimal policy turns out to be ``pure'', leveraging either only communication or only search, depending on their relative costs.

\subsection{Main Contributions}
This paper advances the theoretical understanding of AI-powered recommendation systems driven by user–agent interaction. We make three primary contributions. First, we develop a novel model of user–agent interaction that accounts for the communication cost of preference elicitation and the search cost of evaluating recommendations, with costs measured information-theoretically (via KL divergence and set size entropy).

Second, we derive a high-dimensional approximation of the problem under an asymptotic regime where the number of features ($d$) becomes large, i.e., $d\rightarrow \infty$. By establishing a Large Deviation Principle (LDP) for the maximum utility induced as a function of recommendation set size, we reduce the complex stochastic optimization problem to a deterministic optimization problem, yielding an explicit characterization of the optimal system design. We note that the large-$d$ regime is reminiscent of the architecture of modern recommendation systems, where user preferences and product features are represented as high-dimensional embeddings.

Third, using the analytic solution for the asymptotic limit, we characterize the system's performance under the optimal interaction policy and identify distinct operational regimes that arise from the relative magnitudes of the communication and search costs. Next, we provide a brief overview of our model, followed by the key insights derived from our work.

\subsubsection{Model:} We introduce a stylized model which formalizes the user–agent interaction in a $d$-dimensional spherical feature space, where both user preferences and products are represented as unit vectors.
The utility derived from a product is its alignment with the user's true preference, measured by the dot product. User preferences are uniformly distributed, and every point on the surface of the $d$-dimensional sphere is a product. The interaction unfolds in three stages that model the costly exchange of information between the user and the agent, with information theory-inspired quantification of costs. First, the user transmits a noisy message with precision $\kappa$, incurring a communication cost equal to the KL-divergence between the agent’s prior and posterior beliefs of the user's preference vector, multiplied by a communication cost parameter $\lambda_c$. Second, the agent generates a recommendation menu of size $n$, modeled as independent draws from a sampling distribution. Finally, the user identifies the utility-maximizing item from the menu, incurring a search cost equal to the logarithm of the menu size multiplied by a search-cost parameter $\lambda_s$. We frame the interaction as a collaborative process where the shared goal is to maximize the user's expected payoff (utility from the chosen item minus the search and communication costs) by jointly optimizing the precision of the user's message $\kappa$ and the size of the recommendation set $n$ provided by the agent. We consider two recommendation set sampling strategies: (i) posterior sampling, where the agent samples items from the posterior over the user's preferences vector, and (ii) optimized ``tilted'' sampling, where recommendations are drawn from a modified distribution that places additional, optimally chosen weight (a “tilt”) on the user’s communicated preference.

\subsubsection{Results and Insights:} In this section, we describe our analytic results characterizing the optimal message precision ($\kappa$) and recommendation set size ($n$). More extensive discussion is provided in Section \ref{sec:asymptotic_regimes} and Section \ref{sec:finite_d_regimes}. We first discuss the setting where recommendations are sampled from the posterior distribution over user preferences, mirroring what one may expect a Gen AI model to do by default, in the absence of any intentional tuning.

\paragraph{Joint optimization and hybrid communication-search regime:}

\begin{wrapfigure}{r}{0.35\textwidth}
    \centering
    \includegraphics[width=0.9\linewidth]{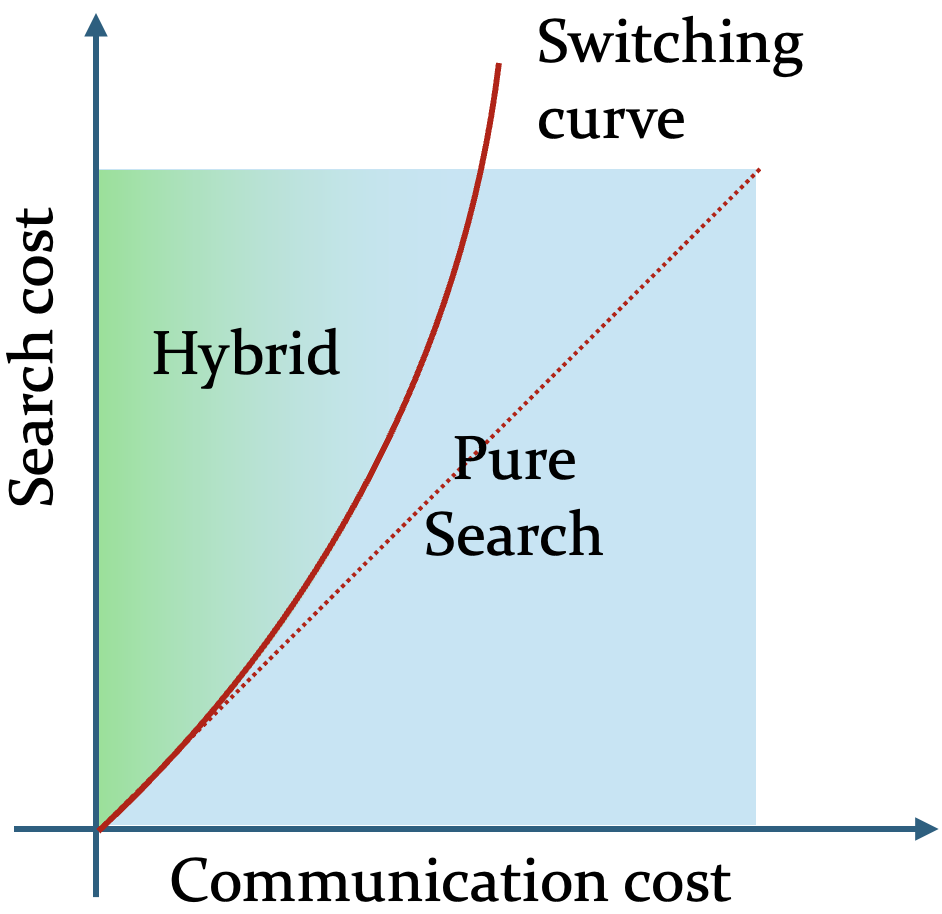}
    \setlength{\belowcaptionskip}{-9pt}
    \caption{Qualitative illustration of the optimal interaction regimes under posterior sampling.
    }
    \label{fig:hybrid_regime}
\end{wrapfigure}

A central insight from our analysis is that optimal performance requires balancing message precision and recommendation set size according to the relative magnitudes of communication and search costs. We identify a single switching curve in the cost–parameter space.

On one side of this curve, the optimal policy jointly optimizes communication and search; on the other side, it relies solely on search. In particular, when the search and communication costs scale as $\lambda_s \sim 1/d$ and $\lambda_c \sim 1/d$, there exists a nontrivial region -- bounded by this switching curve -- in which both costs are economically significant (see the Hybrid region in Fig.~\ref{fig:hybrid_regime}). Within this hybrid regime, optimal performance requires \textit{joint optimization}: nontrivial communication (positive message precision) together with non-degenerate recommendation sets (more than a single item).

\textit{Intuition behind the scaling laws under the hybrid regime:} The high-dimensional approximation also allows us to characterize the scaling laws for message precision $\kappa$ and recommendation set size $n$ with respect to the feature dimension $d$.
\begin{itemize}
    \item \textit{Scaling of message precision:} Our model assumes that the user preferences are uniformly distributed on a $d$-dimensional unit hypersphere.
    In this setting, the total uncertainty (entropy) about a user's preference scales linearly with the dimension $d$. For communication to be meaningful, the information a user provides must overcome an $\Omega(1)$ fraction of this uncertainty. Our results confirm that in the hybrid regime, the user communicates a fraction of their total preference information, i.e., providing an amount of information that scales linearly with $d$. Intuitively, this is similar to the user perfectly specifying a subset of their preference ``features'' while leaving the rest unspecified.

    \item \textit{Scaling of recommendation set size:} When a user provides partial information (or no information), a sizeable product subspace remains where their preferred item might reside.
    To account for these unspecified features, the agent must provide multiple recommendations to ensure adequate coverage. Intuitively, if each unspecified attribute can vary over a few qualitatively distinct levels (e.g., low/medium/high price, sporty/casual/professional style), then the agent would need to provide a recommended product for each possible combination of these levels for unspecified features, leading to an optimal recommendation set size which grows exponentially in $d$.

    We note that the optimal recommendation set size scales as $n = \exp(\alpha d)$ with typically a small exponent $\alpha$. As a result, the number of recommended items is much smaller than the size of the product space. As a concrete example, for $d=15$, we find that $\alpha = 0.15$ for plausible communication and search costs, yielding an optimal recommendation set size of $n = 15$, which is realistic. Note that the dimension $d$ in our model should be interpreted as capturing the effective preference complexity relevant for the interaction so $d=15$ may be reasonable for many product categories. We further study a model extension that allows different features to carry different weights (see Appendix \ref{sec:weighted_prefer}).
\end{itemize}

\paragraph{Recommendations from tilted distribution:}
Next, we consider a more sophisticated recommendation sampling policy: importance sampling of product recommendations, which we call sampling from a ``tilted'' distribution, loosely inspired by the possibility of creating intentionally tuned Gen AI tools such as ChatGPT Shopping Research mode. The core insight is that the agent can optimally manage (in a high-dimensional setting) the trade-off between exploiting the user's message and diversifying across uncommunicated features by adjusting a single, deterministic ``tilt'' parameter; as a result we call this approach \textit{tilting}. Here, the tilt parameter acts as a design lever that the agent must calibrate based on both the message precision as well as the recommendation set size.

\begin{wrapfigure}{r}{0.35\textwidth}
    \centering
    \includegraphics[width=0.9\linewidth]{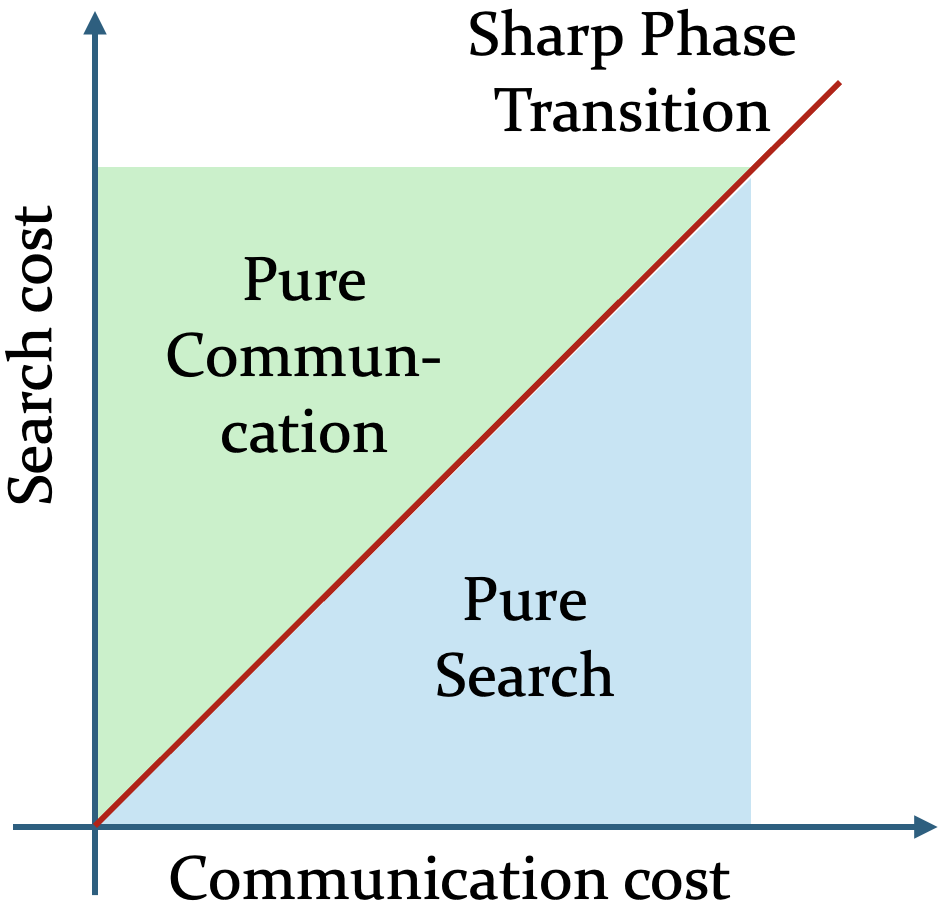}
    \setlength{\belowcaptionskip}{-8pt}
    \caption{Qualitative illustration of the optimal interaction regimes under tilted sampling.}
    \label{fig:tilted_regime}
\end{wrapfigure}
Our analysis demonstrates that optimal tilting outperforms direct posterior sampling, and increases the return to communication by placing greater weight on the user’s message. Furthermore, when all control parameters (message precision, recommendation set size, and the tilt parameter) are jointly optimized, perhaps surprisingly, a sharp phase transition emerges for large $d$. The optimal policy is pure, relying entirely on either communication or search, whichever is more cost-effective, with no hybrid regime arising. This behavior is qualitatively illustrated by the diagonal switching curve in Fig.~\ref{fig:tilted_regime}: on one side of a diagonal boundary in the cost-parameter space the policy relies entirely on communication, while on the other it relies entirely on search.

Two natural questions emerge: First, why does optimal tilting place greater weight on the user’s message than posterior sampling? At one extreme, when there is only a single recommendation ($n=1$), the optimal action is to recommend the item corresponding exactly to the user’s message (the maximum possible tilting); there is no room to compensate for the uncertainty in preferences. At the other extreme, when the number of recommendations is very large ($n\to \infty$), sampling from the posterior distribution (no tilting) works well, since with high probability at least one sampled item lies very close to the user’s true type, and additional tilting would only distort this distribution and reduce expected utility. For finite $n>1$, we are in an intermediate situation and appropriate tilting towards the user's message helps by acting as a form of regularization (avoiding overfitting to the posterior distribution).

Second, why does the optimal policy under large $d$ effectively use only one of the two modes of interaction when tilting is allowed? The key intuition is that tilting provides a lever for controlling how the user’s message is translated into recommendations. By tuning this lever, in our model for large $d$, the system is able to fully exploit the user's communication, and also make best use of product recommendations. As a result, the optimal policy favors a single mode, specifically, the mode which costs less per unit (recall that our model quantifies both costs along information theoretic lines).

We note that posterior sampling and tilting represent fundamentally different agent capabilities, not simply alternative methods with different performance. Posterior sampling captures a \textit{constrained} agent architecture, typical of modular or black-box generative systems, where the agent cannot modify its underlying posterior distribution. In this setting, balancing communication and search is essential to improve performance. In contrast, the tilted distribution represents an \textit{optimized} agent with the architectural flexibility to modify its internal sampling distribution to directly target the system objective. This additional design freedom enables both higher performance and qualitatively different interaction regimes. At the same time, realizing these gains requires a more sophisticated training and tuning pipeline, although such complexity may be justified in high-stakes applications such as e-commerce. Importantly, only a single-dimensional tilting parameter is needed, making the implementation analogous to choosing a context-dependent temperature parameter in modern generative systems.

\subsection{Related Work}
\paragraph{AI-Assisted Decision Making:} With the advent of chatbots and AI agents, there is huge interest in works that examine how human decision-making interacts with outputs from algorithmic agents. \citet{kleinberg2018human} show that machine learning predictions can augment human judgment in socially sensitive contexts, while \citet{angelova2023algorithmic} study settings where individuals retain discretion over whether to accept or reject algorithmic recommendations. In the medical domain, \citet{agarwal2023combining} provide experimental evidence that combining human expertise with AI predictions improves diagnostic accuracy, though the extent of improvement depends critically on how information is communicated and the cognitive effort required from experts.

Another strand of research investigates the cognitive underpinnings of Human–AI interaction. \citet{vasconcelos2023explanations} show that overreliance on AI persists even when predictions are accompanied by explanations, and argue that this is not simply a cognitive bias but rather a strategic choice rooted in cost–benefit reasoning: decision-makers implicitly weigh the cognitive effort of verification against the ease of acceptance. In a similar spirit, \citet{boyaci2024human} use a rational inattention framework (similar to our work) to study how machine predictions affect human decision-makers with limited cognitive capacity. They find that while AI assistance improves accuracy, it can also increase error rates and induce users to exert more cognitive effort.

Closest to our work is \citet{castro2023human}, who study the trade-off between output fidelity and communication cost under single-output recommendations. They show that preference heterogeneity can introduce substantial bias when only one recommendation is presented. In contrast, we argue that offering multiple recommendations can reduce this bias, as it allows users to choose from a broader set and better align outcomes with their preferences. Similarly, \citet{liang2025artificial} analyze delegation to AI “clones” and find that in high-dimensional settings, noise degrades performance and humans may outperform. Our framework shows that when multiple outputs are available, users can mitigate this risk by selecting the best fit, and further, that in high dimensions it is optimal for users to communicate richer information to ensure credible recommendations.

\paragraph{Costly Information Acquisition:} From an economic and information-theoretic perspective, our model builds on the rational inattention (RI) framework \citep{sims2003implications, sims2005rational}, which has become a foundational approach to modeling bounded rationality under costly information processing. The central insight of RI is that decision-makers optimally allocate limited attention across signals, trading off informativeness against the cost of acquiring information. This framework has been applied extensively in macroeconomics and finance, including monetary policy design \citep{sims2006rational, mackowiak2009optimal}, portfolio choice \citep{zhong2022optimal}, and multivariate information acquisition problems \citep{miao2022multivariate}. Closely related are approaches in information theory, such as rate–distortion methods \citep{tishby2010information}, which also study trade-offs between informativeness and cost. We also refer the reader to the review \citep{mackowiak2023rational} for more details. While most prior work applies RI to single-agent decision problems, few works \citep{castro2023human,boyaci2024human}, alongside our own, have used the RI framework to study the cognitive costs and trade-offs in human–AI interaction.

\paragraph{Product Recommendation:} Research on product recommendation has traditionally centered on collaborative filtering \citep{koren2021advances, su2009survey} and content-based methods \citep{thorat2015survey}, with hybrid approaches combining user-item interaction histories, product features, and contextual signals \citep{sarwar2001item, linden2003amazon}. These foundational methods, widely deployed in e-commerce platforms like Amazon and Netflix, rely on structured data to exploit similarities across users or products.
A parallel stream of research in human-computer interaction and information retrieval has focused on the concept of the \textit{selection set}, the curated subset of options from which a user makes their final choice. This literature recognizes that the value of a recommendation system lies not only in the relevance of individual items but also in the composition and size of the presented set. In particular, diversity can improve coverage under preference uncertainty \citep{ziegler2005improving}, while overly large sets risk cognitive overload and degrade decision quality \citep{iyengar2000choice}. This perspective, that the recommendation task fundamentally involves managing a user’s attention budget, directly motivates our theoretical treatment of communication and search costs.

The recent advent of Large Language Models (LLMs) has marked a new phase in recommendation research, introducing powerful new capabilities. LLMs can act as natural-language interfaces, allowing users to express complex preferences conversationally rather than through simple clicks or ratings \citep{zhang2023chatgpt}. Furthermore, pre-trained on vast corpora, they serve as knowledge-rich priors that encode deep semantic relationships, enabling effective zero-shot and few-shot recommendation \citep{bao2024large, wu2024survey}. Architectures such as retrieval-augmented generation (RAG) \citep{lewis2020retrieval} now explicitly separate the retrieval of a candidate set from the generation of the final recommendation, mirroring the two-stage structure of our model. Our contribution complements these developments by offering a theoretical framework that characterizes how users optimally trade off communication effort with reliance on the AI-generated recommendation set, an aspect largely absent from algorithmic studies of LLM-based recommenders.

\section{Model Description}
We consider a market consisting of a continuous space of products, each represented by a feature vector $\boldsymbol \theta \in \mathcal{S}^{d-1}$, where $\mathcal{S}^{d-1}$ is the surface of a unit sphere in $\mathbb R^d$, that is, $\mathcal{S}^{d-1} = \{ \boldsymbol \theta \in \mathbb R^d : \|\boldsymbol \theta \|_2 = 1 \}$. Users are similarly characterized by a preference vector $\mathbf h \in \mathcal{S}^{d-1}$, which encodes their most preferred product profile or \textit{true preferences}. We assume user preferences follow a uniform distribution over $\mathcal{S}^{d-1}$, giving us $\mathbf h\sim p(\mathbf h) = C_d(0)$, where $1/C_d(0)$ is the surface area of $\mathcal{S}^{d-1}$.

The utility that a user derives from selecting a product with feature vector $\boldsymbol \theta$ depends on how well the product matches their preference vector $\mathbf h$, and is given by the dot product between the two vectors, i.e., $u(\boldsymbol \theta, \mathbf  h) = \langle \mathbf h , \boldsymbol \theta \rangle$. Note that the utility of a product with feature $\boldsymbol \theta$ can also be expressed as the negative squared $\ell_2$-distance from the user's preferences $\mathbf h$, i.e., $u(\boldsymbol\theta , \mathbf h) =1 -\frac{1}{2}\|\boldsymbol \theta -\mathbf h\|^2_2$.

\textit{Motivation behind the feature space model:}
In our model, the space of products is represented through a set of underlying features (such as color, or size) that together determine user preferences. We assume that this feature space is characterized by a Pareto frontier, where each point corresponds to a product that is most preferred by some subset of users. Intuitively, any product that is strictly dominated, i.e., less preferred by all users compared to another, is redundant and can be excluded from the space. For a tractable stylized representation, we model this Pareto frontier as a high-dimensional isotropic space $\mathcal{S}^{d-1}$.

A complementary motivation for modeling the feature space as $\mathcal{S}^{d-1}$ arises from a tokenized representation of products and preferences, commonly employed in modern recommendation and language models. In such systems, product descriptions, and similarly, users' preferences, comprising textual, visual, and numerical attributes, are embedded as high-dimensional vectors in a shared semantic space. The relevance of a product to a user’s preferences is then primarily determined by the alignment between their embeddings, typically measured by cosine similarity. Such embedding-based representations form the foundation of large-scale recommendation systems and language models (see, e.g., \citealt{barkan2016item2vec}).

\noindent\textbf{User–Agent interaction:} We model the user–agent interaction as a three-stage process.
\vspace{1em}

\begin{tikzpicture}[
  box/.style={draw, rounded corners=2pt, minimum width=1.4cm, minimum height=0.7cm, align=center, font=\small\bfseries}, arrow/.style={-{Latex[length=2mm]}, thick, shorten >=1pt, shorten <=1pt}, node distance=3.2cm ]

\node[box] (human) {User ($\mathbf{h}$)};
\node[box, right=3.3cm of human] (ai) {Agent};
\node[box, right=5cm of ai] (ri) {User};
\node[right=1.9cm of ri] (choice) {$\boldsymbol\theta^*$};

\draw[arrow] (human) -- node[midway, above, font=\small] {$\mathbf{m} \sim p_\kappa(\mathbf{m}|\mathbf{h})$} (ai);
\draw[arrow] (human) -- node[midway, below, font=\small] {Message} (ai);
\draw[arrow] (ai) -- node[midway, above, font=\small] {$\boldsymbol \theta_1, \dots, \boldsymbol \theta_n \sim q_{\kappa}(\mathbf{h}|\mathbf{m}), \text{ i.i.d.}$} (ri);
\draw[arrow] (ai) -- node[midway, below, font=\small] {Recommendations} (ri);
\draw[arrow] (ri) -- (choice);
\draw[arrow] (ri) -- node[midway, below, font=\small] {Search} (choice);
\end{tikzpicture}
\begin{itemize}
    \item[(i)] \textit{Communication:} The interaction starts with the user transmitting a message (or context) $\mathbf{m}$ to the agent, where the message is generated stochastically given their feature $\mathbf{h}$ via a conditional distribution $p_{\kappa}(\mathbf{m}|\mathbf{h})$. This formulation captures the practical reality that users often find it cumbersome to articulate their complete preferences, as such, the message $\mathbf{m}$ may be incomplete, noisy, or only partially informative about the underlying preference vector $\mathbf{h}$. The stochasticity in $p_{\kappa}(\cdot|\mathbf{h})$ thus reflects both communication noise and inherent ambiguity in self-reported preferences. In this work, we assume that the message provided by the user follows a von Mises Fisher distribution (vMF) \citep{mardia2009directional}, given by
    \begin{equation}
    \label{eq:vMF_dist}
        p_{\kappa}(\mathbf{m}|\mathbf{h}) = C_d(\kappa) \exp \big(\kappa \langle\mathbf h, \mathbf m\rangle\big),
    \end{equation}
    where $C_d(\kappa)$ is the normalization constant, with a known closed form expression \citep[Chapter 9]{mardia2009directional}. The parameter $\kappa$ is typically known as the \textit{concentration parameter} as it governs how closely the user's message aligns with the true preference vector $\mathbf{h}$, where a larger $\kappa$ implies that $\mathbf{m}$ is concentrated closer to $\mathbf{h}$. As such, we refer to $\kappa$ as the \textit{message precision}. When $\kappa=0$, the distribution $p_{0}(\cdot|\mathbf{h})$ matches the uniform distribution $p(\cdot)$.

    \item[(ii)] \textit{Recommendation:} The agent interprets the message provided by the user, and uses Bayes' rule to obtain a posterior $q_\kappa(\mathbf{h}|\mathbf{m})$ using $p(\mathbf{h})$ as the prior, which results in
    $q_\kappa(\mathbf{h}|\mathbf{m}) = C_d(\kappa) \exp\big(\kappa \langle \mathbf{h},\mathbf{m}\rangle\big)$. Using the inferred posterior $q_{\kappa}(\mathbf{h}|\mathbf{m})$, the agent generates a menu of $n$ product recommendations $\{\boldsymbol\theta_1, \dots, \boldsymbol \theta_n\}$. In Section \ref{sec:results_insights}, we study the setting where $\boldsymbol{\theta}_i \sim q_{\kappa}(\cdot|\mathbf{m})$ for all $i\in[n]$, i.i.d. This is a stylized but natural modeling assumption motivated by modern Gen AI–based recommender systems: the agent produces recommendations by sampling from its posterior belief over the user’s preferences. Conceptually, this reflects the fact that autoregressively trained models act as implicit simulators of complex conditional distributions: having internalized rich latent structure during training, they can generate samples from the conditional distribution of “what the user might like” given the prompt, without explicitly constructing or optimizing a posterior.

    \textit{Recommendations from tilted distribution:} In a later part (Section \ref{sec:tilted_dist}), we consider a more sophisticated agent that optimally tilts (modifies) the posterior distribution to maximize the expected utility of the best-performing item in the recommendation menu. This approach explicitly accounts for the fact that the user’s utility is governed by the maximum of the $n< \infty$ recommended items. Consequently, the tilted distribution and the recommendation set size $n$ are jointly optimized to maximize the user’s expected utility.

    \item[(iii)] \textit{Search:} After receiving the recommendations $\{\boldsymbol\theta_1, \dots, \boldsymbol \theta_n\}$, the user evaluates each product to identify the best match based on their true preferences, $\mathbf{h}$, i.e., the user chooses the item
    \begin{align*}
        \boldsymbol \theta^* = \arg\max_{\boldsymbol \theta_i} \ \langle \mathbf h , \boldsymbol \theta_i \rangle.
    \end{align*}
    We assume that the user perfectly identifies the item that maximizes this alignment among the $n$ recommendations.
\end{itemize}

\noindent \textbf{Objective:} We now formalize the utility the user derives from a chosen product and the costs they incur during the interaction process. Throughout the remainder of the paper, we refer to the pair $(\kappa,n)$ as the \textit{interaction policy}.
\begin{itemize}
    \item[(i)]

    \textit{Utility from chosen recommendation:} As mentioned before, given the set of recommendations $\{\boldsymbol\theta_1, \dots, \boldsymbol \theta_n\}$, the user chooses the product that best matches their preferences. As such, the product utility received by the user is given by $\max_i \langle \mathbf h , \boldsymbol \theta_i \rangle$.

    \item[(ii)]
    \textit{Communication cost:}
    We assume that when the agent acquires information (message $\mathbf{m}$) about a user's features $\mathbf{h}$ to update their initial belief (the prior $p(\mathbf{h)}$) to a more informed one (the posterior $q_{\kappa}(\mathbf{h}|\mathbf{m})$), communication cost accrues to the user. We consider the cost to be linearly proportional to the \textit{information gain}, quantified by the KL-divergence between the prior and posterior distributions. As such, the communication cost is $\lambda_c D_{\mathrm{KL}}\big(q_{\kappa}(\mathbf{h}|\mathbf{m}) \| p(\mathbf{h})\big)$, where $\lambda_c$ is the communication cost parameter.

    \item[(iii)]
    \textit{Search cost:}  The search cost represents the effort a user expends to find the best product among a set of $n$ recommendations. Similar to the communication cost, we assume the cost is proportional to the information gain.
    In particular, we assume that at the start of the search process, a user views all $n$ product recommendations as equally likely to be the best, i.e., the user's prior on the index of the best item follows a uniform distribution over the set of indices $[n]$. Once the user determines the best product, their belief becomes a certainty ($\boldsymbol \theta_i$ is the best with probability either $0$ or $1$), and the total information gain during the search process equals the entropy of the uniform distribution over $[n]$, i.e., $\log n$. Motivated by this, we assume the user's search cost is $\lambda_s \log n$, where $\lambda_s$ is the search cost parameter.
\end{itemize}
Overall, the objective of an average user is given by
\begin{align*}
    \mathcal{P}_d(\kappa,n) := \mathbb E_{\mathbf h \sim p(\mathbf h)} \mathbb E_{\mathbf m \sim p_\kappa (\mathbf m | \mathbf{h})}\mathbb E_{\boldsymbol \theta_i \sim q_\kappa (\mathbf h | \mathbf{m}),  \textup{ i.i.d.}}\Big[ \max_{i} \ \langle \mathbf h , \boldsymbol \theta_i \rangle - \lambda_s \log n - \lambda_c D_{\mathrm{KL}}\big(q_{\kappa}(\mathbf{h}|\mathbf{m}) \| p(\mathbf{h})\big) \Big].
\end{align*}
Our objective in this work is to characterize the interaction policy $(\kappa, n)$ that maximizes $\mathcal{P}_d(\kappa,n)$.

\subsection{Decomposition of Preference and Recommendation for a given Message}
\label{sec:decomposition}
Any vector $\mathbf{x} \in \mathcal{S}^{d-1}$ admits a unique decomposition into components parallel and orthogonal to a given unit vector $\mathbf{m} \in \mathcal{S}^{d-1}$. This decomposition is expressed as $\mathbf{x} = \langle \mathbf{x} , \mathbf{m} \rangle \mathbf{m} + \sqrt{1-\langle \mathbf{x} , \mathbf{m} \rangle^2} \mathbf{x}_{\perp}$, where $\mathbf{x}_{\perp} \in \mathcal{S}^{d-1}$ is orthogonal to $\mathbf{m}$ (i.e., $\langle \mathbf{x}_{\perp} , \mathbf{m} \rangle = 0$).  We leverage this to decompose a random variable distributed as per the posterior over user preferences $\overline{\mathbf{h}}$, and the set of recommendations $\{ \boldsymbol\theta_1, \dots, \boldsymbol\theta_n\}$ conditioned on the message $\mathbf{m}$. For notational simplicity (since preferences $\overline{\mathbf{h}}$ lead to the same expected user utility), we drop the bar and denote the posterior preference vector as $\mathbf{h}$.

\textit{Decomposition of posterior on user preferences:} Define the \textit{message fidelity} $W = \langle \mathbf{h}, \mathbf{m} \rangle$, which quantifies the alignment between the preference $\mathbf{h}\sim q_{\kappa}(\cdot|\mathbf{m})$ and the message $\mathbf{m}$. We can write
\begin{equation}
\label{eq:h_decomposition}
    \mathbf{h} = W \mathbf{m} + \sqrt{1-W^2} \mathbf{Y}.
\end{equation}
The vector $\mathbf{Y}$ represents the \textit{uncommunicated component} of $\mathbf{h}$ (the part not captured by the message $\mathbf{m}$), and is orthogonal to $\mathbf{m}$, i.e., $\langle \mathbf{Y} , \mathbf{m} \rangle = 0$. Under the vMF posterior, the random variables $W$ and $\mathbf{Y}$ are independent. Note that a larger message precision $\kappa$ induces stronger concentration of $W$ near $1$. More details on properties of $W$ and $\mathbf{Y}$ are provided in Lemma \ref{lem:prelim_results}.

\textit{Decomposition of recommendations:} Each recommendation $\boldsymbol{\theta}_i$ is drawn independently from the same posterior distribution $q_{\kappa}(\mathbf{h}|\mathbf{m})$. Consequently, we have an analogous decomposition for it
\begin{align}
\label{eq:util_representation}
    \boldsymbol\theta_i = W_i \mathbf{m} + \sqrt{1-W_i^2} \mathbf{Y}_i, \ \text{ and } \ \langle \mathbf{h} , \boldsymbol{\theta}_i \rangle = W W_i + \sqrt{1-W^2} \sqrt{1-W_i^2} X_i,
\end{align}
where $(W_i,\mathbf{Y}_i)$ has the same distribution as $(W,\mathbf{Y})$,  $X_i := \langle \mathbf{Y} , \mathbf{Y}_i \rangle$, and we use $\langle \mathbf{m} , \mathbf{Y} \rangle = \langle \mathbf{m} , \mathbf{Y}_i \rangle = 0$.

\begin{lemma}
\label{lem:prelim_results}
For the message and product distribution, we have the following results:
\begin{enumerate}[label=(\roman*)]
    \item \label{lem:fidelity_dist} \emph{Distribution of $W$ and $\mathbf{Y}$:} The distribution function for $W$ is given by
    \begin{align}
    \label{eq:cos_dist}
        p_{\kappa,d}(w) = \tilde{C}_d(\kappa) e^{\kappa w} \left(1-w^2\right)^{\frac{d-3}{2}}, \ \text{ for } w \in [-1,1].
    \end{align}
    where $\tilde{C}_d(\kappa)$ is the normalization constant. Further, the uncommunicated component $\mathbf{Y}$ is uniformly distributed over the space $\mathcal{S}(\mathbf{m}) = \{ \mathbf y \in \mathcal{S}^{d-1}: \langle \mathbf{m} , \mathbf{y} \rangle = 0\}$.

    \item \label{lem:X_independence}  \emph{Independence of orthogonal components:} The set of random variables $\{X_1, \cdots, X_n\}$ are mutually independent, and follow the distribution $p_{0,d-1}(x)$, where $p_{\kappa,d}(\cdot)$ is as given in Eq. \eqref{eq:cos_dist}.

    \item \label{lem:kl_div_simplifiction} \emph{Communication cost simplification:}  The KL-divergence between $p(\mathbf{h})$ and $q_\kappa(\mathbf{h} | \mathbf{m})$ is
    \begin{align*}
        \EE_{\mathbf{h},\mathbf{m}} \big[D_{\emph{KL}}\big(q_{\kappa}(\cdot | \mathbf{m}) \| p(\cdot)\big) \big] =  \kappa \EE[W] -\log \frac{C_d(0)}{C_d(\kappa)},
    \end{align*}
    where the distribution of $W$ is itself a function of $\kappa$.
\end{enumerate}
\end{lemma}
The results in Lemma \ref{lem:prelim_results} characterize the distributional and information-theoretic properties of the underlying random variables. The proof utilizes standard results in directional statistics (see \citep{mardia2009directional}) regarding the marginal distribution of the vMF distribution, and is provided in Appendix \ref{sec:proof_lem_prelims}. Next, using Eq. \eqref{eq:util_representation} and Lemma \ref{lem:prelim_results}\ref{lem:kl_div_simplifiction}, we have
\begin{align*}
    \mathcal{P}_d(\kappa,n) &= \mathbb E \Big[\max_i \big\{  W W_i + \sqrt{1-W^2} \sqrt{1-W_i^2} X_i \big\}\Big] - \lambda_s \log n - \lambda_c \Big(\kappa \EE[W] - \log \frac{C_d(0)}{C_d(\kappa)}\Big).
\end{align*}

\textbf{Alternative Policies:} To study if there are gains from jointly optimizing communication and search, we compare against policies that use only one of them.
\begin{enumerate}
    \item \textit{Pure Search:} Under this policy, the user does not communicate their preferences, they jump straight to choosing between recommendations received.
    Mathematically, this amounts to setting $\kappa=0$ in the joint optimization problem. The agent’s posterior collapses to the prior, recommendations are samples from $p(\cdot)$ and the expected payoff is $\mathcal{P}_d(0,n)$.

    \item \textit{Pure Communication:} In this case, we consider a setting where the agent is restricted to recommending a single item, i.e., $n=1$. Under this policy, the user relies on providing an informative message to ensure that the agent delivers a single high-quality recommendation. Formally, when $n=1$, the objective simplifies to $\mathcal{P}_d(\kappa,1)$.
\end{enumerate}

\section{Main Results and Insights under Posterior Sampling Scheme}
\label{sec:results_insights}
While the optimization problem $\max_{\kappa, n} \mathcal{P}_d(\kappa, n)$ can be solved numerically, such computation offers limited structural insight. To gain structural insight, we turn to a high-dimensional asymptotic approximation, which serves as a tractable surrogate that captures the dominant interactions between key parameters. It reveals the fundamental scaling laws and operational regimes governing the optimal user–agent interaction.

\subsection{Intuitive Explanation behind High-Dimensional Approximation}
\label{sec:intuitive_high_dim}
In modern recommendation systems, both users and products are described by high-dimensional embeddings. Analyzing the regime where the feature dimension $d$ is large not only mirrors practical settings but also enables powerful asymptotic simplifications to reveal key tradeoffs between communication and search costs. As $d$ grows, the alignments $\langle \mathbf{h}, \boldsymbol{\theta}_i \rangle$'s concentrate sharply around their mean. Deviations of constant order away from this typical value occur with probabilities that decay exponentially in $d$. These rare but utility-relevant deviations are naturally characterized by a Large Deviation Principle (LDP) with an associated rate function that quantifies the exponential cost of achieving a given alignment level relative to its typical value. This perspective allows us to replace the stochastic maximization over a finite set of randomly generated recommendations with a deterministic optimization problem defined by the LDP rate function. In particular, the maximum utility is governed by the most likely extreme alignment achievable at exponential scale, rather than by typical fluctuations. The resulting approximation captures the asymptotically dominant alignment outcomes that govern the maximum utility in high dimensions.

\paragraph{Intuition behind scaling of $\kappa$ and $n$:} In order to conduct the high-dimensional approximation for the optimization problem $\max_{\kappa, n} \mathcal{P}_d(\kappa, n)$, we first examine the product utility received by the user as $d\rightarrow\infty$. Consider the alignment of a recommendation $\boldsymbol \theta_1 \sim q_{\kappa}(\cdot|\mathbf{m})$ with respect to the message $\mathbf{m}$, denoted by $W_1=\langle \boldsymbol \theta_1, \mathbf{m}\rangle$. The induced one-dimensional density of $W_1$ is $p_{\kappa,d}(w) \propto e^{\kappa w} \left(1-w^2\right)^{\frac{d-3}{2}}$. The exponential term $e^{\kappa w}$ pulls mass towards $w=1$ when $\kappa$ is large, while the spherical term $\left(1-w^2\right)^{\frac{d-3}{2}}$ pulls mass toward $w=0$ when $d$ is large. If $\kappa$ is fixed while $d\rightarrow\infty$, the spherical term dominates and $w$ tends to $0$. On the other hand, if $\kappa$ grows faster than $d$, the exponential term $e^{\kappa w}$ dominates and $w$ concentrates near $1$. In contrast to the two extremes, the exponential and the spherical are in balance when $\kappa$ scales linearly with $d$. To capture this, we set $\kappa=\frac{\rho}{1-\rho^2}(d-3)$ in the analysis (see Proposition \ref{prop: utility_approx}). Here $\rho$ turns out to be the mode of the distribution $p_{\kappa,d}(w)$, and $W, W_i \sim p_{\kappa,d}(w)$ concentrate at $\rho$ as $d \to \infty$.

Next, consider $n$ candidate products. The LDP implies that the probability of observing an unusually high alignment for any single product decays exponentially in $d$, say $e^{-d\alpha}$ for some $\alpha > 0$ for a  given degree of alignment. When searching over many products, the most preferred one will correspond to such a rare alignment, but to have a meaningful chance of observing it, the number of candidates must grow fast enough to offset this exponential decay. In particular, we require $ne^{-d\alpha}$ to remain non-trivial, for some $\alpha>0$, that is $n \propto \exp(d\alpha)$. Intuitively, exponential in $\Theta(d)$ recommendations are needed to address preferences which are unknown in $\Theta(d)$ dimensions.

\paragraph{Intuition behind scaling of cost parameters:} We next analyze the communication and search costs. For $\kappa=\frac{\rho}{1-\rho^2}(d-3)$, the expected information gain, measured by the KL-divergence,
grows linearly with $d$ (see Proposition \ref{prop: kl_approx}). Similarly, for $n\propto \exp(d\alpha)$, the search cost, proportional to $\log n$, also scales linearly in $d$. To keep the overall optimization meaningful in this regime, the communication and search cost coefficients $\lambda_c$ and $\lambda_s$ must scale inversely with $d$, i.e., $\propto 1/d$. Note that if the cost structure changes, the scaling must adjust accordingly. For example, if search cost were proportional to $n$, then $\lambda_s$ would need to decay exponentially, i.e., $\propto e^{-\alpha d}$.

\subsection{Theoretical Results and Insights for \texorpdfstring{$d\rightarrow \infty$}{d->infty}}
\label{sec:limiting_d_results}

We start by providing a high-dimensional approximation for $\mathcal{P}_d(\kappa,n)$.
\begin{proposition}[Product utility]
\label{prop: utility_approx}
Suppose $\rho \in [0,1)$ and $\alpha \geq 0$ are fixed and define:
\begin{align*}
    f(\rho,\alpha) := \max_{w, x \in (-1,1)} & \ \ \ \rho w + \sqrt{1-\rho^2} \sqrt{1-w^2} x \ \text{ such that } I_{\rho}(w,x) \leq \alpha,
\end{align*}
where $I_{\rho}(w,x)$ is the large deviations rate function, given by
\begin{align}
\label{eq:rate_function}
    I_{\rho}(w,x) = -\frac{\rho(w-\rho)}{1-\rho^2} - \frac{1}{2} \log(1-w^2)-\frac{1}{2} \log(1-x^2) + \frac{1}{2} \log(1-\rho^2).
\end{align}
Suppose the message precision $\kappa = \frac{\rho}{1-\rho^2}(d-3)$ and the recommendation set size $n = \lfloor e^{d\alpha}\rfloor$, then, the corresponding expected utility received by the user satisfies
\begin{align*}
    \mathbb E \left[\max_i \Big\{  W W_i + \sqrt{1-W^2} \sqrt{1-W_i^2} X_i \Big\}\right] = f(\rho,\alpha) + \mathcal{E}_{\ref{prop: utility_approx}}, \ \text{ where }  \ |\mathcal{E}_{\ref{prop: utility_approx}}| \leq K_{\ref{prop: utility_approx}} \sqrt{\frac{\log d}{d}},
\end{align*}
where $W, W_i$ and $X_i$'s are as in Eq. \eqref{eq:util_representation}, and the constant $K_{\ref{prop: utility_approx}}$ depends on $\rho$ and $\alpha$.
\end{proposition}
Proposition \ref{prop: utility_approx} provides a high-dimensional approximation of the expected utility an average user receives, and shows that the approximated expected utility can be obtained by solving a deterministic optimization problem. The parameter $\rho$ denotes the mode of the distribution of $W$, and as the dimension $d$ increases, the distribution of $W$ concentrates sharply around this mode. The result in Proposition \ref{prop: utility_approx} is derived using a LDP for the set of sample pairs $\{(W_i, X_i)\}_{i\in[n]}$. The LDP implies that as the dimension grows, the set of points $\{(W_i, X_i)\}_{i\in[n]}$ lies, with high probability, within the effective support of the joint distribution of $(W_i,X_i)$. This effective support is characterized by sub-level sets of the associated rate function $I_{\rho}(w,x)$ (the rate function corresponding to the distribution of $(W_i, X_i)$), and its boundary determines the extreme realizations that govern the maximum utility. The approximation error of $\tilde{O}\big(d^{-\frac{1}{2}}\big)$ arises from residual stochastic fluctuations in the random variable $W$, where we use the fact that the standard deviation of $W$ is of order $d^{-\frac{1}{2}}$. The proof of Proposition \ref{prop: utility_approx}  is provided in Appendix \ref{sec:proof_prop_util_approx}.

\begin{proposition}[Communication cost]
\label{prop: kl_approx}
Suppose $\rho$ is fixed and let $\kappa = \frac{\rho}{1-\rho^2}(d-3)$. We have,
\begin{align*}
\EE_{\mathbf{h},\mathbf{m}} \big[D_{\mathrm{KL}}\big(q_{\kappa}(\cdot|\mathbf{m}) \| p(\cdot)\big) \big] = \frac{d-2}{2} \log \frac{1}{1-\rho^2} + \mathcal{E}_{\ref{prop: kl_approx}},
\end{align*}
where $|\mathcal{E}_{\ref{prop: kl_approx}}(\rho)| \leq K_{\ref{prop: kl_approx}} \sqrt{\frac{d}{(1-\rho^2)^3}}$, and $K_{\ref{prop: kl_approx}}$ is a universal constant.
\end{proposition}

Proposition \ref{prop: kl_approx} provides an approximation of the KL divergence between the prior on the user's preference and the posterior given the message. It is important to note that the KL divergence scales linearly in the dimension $d$. This makes sense: Recall that the entropy of prior preference distribution (i.e., uniform over $\mathcal{S}^{d-1}$) scales linearly with the dimension $d$. As such, the amount of information that the user needs to provide to meaningfully reduce the agent's uncertainty must also grow proportionally. Conceptually, this linear scaling implies that the user effectively pays a constant price for each dimension of preference they clarify for the agent.

Combining Propositions \ref{prop: utility_approx} and \ref{prop: kl_approx}, we define the following optimization problem, which allows us to state our first theorem:
\begin{align}
\label{eq:joint_opt}
\text{OPT}_{\text{Joint}} := \max_{\rho \in [0,1), \alpha \geq 0} \ \ \Big \{ f(\rho,\alpha) - c_s \alpha + \frac{1}{2} c_c \log (1-\rho^2) \Big\},
\end{align}
where $c_s  \text{ and } c_c$ are constants, and the terms $\rho$, $\alpha$ and $f(\rho,\alpha)$ are as defined in Proposition \ref{prop: utility_approx}.

\begin{theorem}[Asymptotics of Interaction Policy]
\label{thm:asym_opt}
Let $c_s, c_c > 0$ be fixed constants such that the cost coefficients scale as $\lambda_s = \frac{c_s}{d} + o(d^{-1})$ and $\lambda_c = \frac{c_c}{d} + o(d^{-1})$. Let $\rho^*(c_s, c_c)$ and $\alpha^*(c_s,c_c)$ denote the optimal solution to $\mathrm{OPT}_{\mathrm{Joint}}$ given $(c_s, c_c)$, and construct an interaction policy $(\kappa^*_\infty, n^*_\infty)$ as:
\begin{align}
\label{eq:high_dim_interaction_policy}
    \kappa^*_\infty =  \frac{\rho^*(c_s,c_c)}{1-(\rho^*(c_s,c_c))^2} (d-3), \quad \text{ and } \quad n^*_\infty = \lfloor e^{d \alpha^*(c_s,c_c)} \rfloor.
\end{align}
We have the following results:
\begin{enumerate}
    \item \emph{Convergence:} The expected payoff satisfies,
    \begin{equation*}
        \lim_{d\rightarrow \infty} \max_{\kappa , n} \mathcal{P}_d(\kappa,n) = \lim_{d \rightarrow \infty} \mathcal{P}_d(\kappa^*_\infty,n^*_\infty) = \mathrm{OPT}_{\mathrm{Joint}}.
    \end{equation*}
    \item \emph{Monotonicity:} The optimal message precision $\rho^*(c_s, c_c)$ and set size exponent $\alpha^*(c_s,c_c)$ satisfy,
    \begin{align*}
        \frac{\partial \rho^*(c_s, c_c)}{ \partial c_c} \leq 0, &&  \frac{\partial \rho^*(c_s, c_c)}{ \partial c_s} \geq 0, && \frac{\partial \alpha^*(c_s, c_c)}{ \partial c_c} \geq 0, && \frac{\partial \alpha^*(c_s, c_c)}{ \partial c_s} \leq 0
    \end{align*}
    \item \emph{Search-Only Regime:} For every fixed $c_s>0$, there exists a threshold $\bar{c}_c(c_s) > 0$ such that for all $c_c > \bar{c}_c(c_s)$, the optimal policy involves no communication, that is, $\rho^*(c_s, c_c)=0$.
\end{enumerate}
\end{theorem}

Theorem \ref{thm:asym_opt} delivers three key insights into the structure of optimal user–agent interaction in high dimensions. First, the convergence of $\max_{\kappa , n} \mathcal{P}_d(\kappa,n)$ shows that, despite the inherent stochasticity of the recommendation process, the complex finite-dimensional problem admits a deterministic characterization in the limit. This validates the joint optimization problem $\mathrm{OPT}_{\mathrm{Joint}}$ as a tractable proxy for system design when the number of features is large, and allows us to analytically understand the fundamental trade-offs in communication and search. Second, the monotonicity properties reveal a clear and intuitive substitution pattern between communication and search. As communication becomes more expensive, users optimally reduce message precision and compensate by relying on larger recommendation sets, whereas increases in search costs lead to more informative communication and smaller recommendation set size. Finally, the existence of a search-only regime underscores the limits of communication in high-dimensional settings. When communication is sufficiently costly relative to search, the user optimally withholds information and delegates all effort to exploration. The proof of Theorem \ref{thm:asym_opt} is provided in Appendix \ref{app:proof_thm_asym_opt}.

\subsubsection{Asymptotic Regimes of Operation}
\label{sec:asymptotic_regimes}

Our analysis characterizes the asymptotic structure of optimal user–agent interaction under communication and search costs that scale as $\lambda_s = c_s/d$ and $\lambda_c = c_c/d$ with fixed constants $(c_s,c_c)$. Under this specification, the interaction admits a well-defined asymptotic limit characterized by the optimization problem OPT$_{\text{Joint}}$. Varying $(c_s,c_c)$ induces qualitatively distinct regimes of operation with a phase transition between them.

\paragraph{Hybrid Regime (and Joint Optimization)}

The hybrid regime and the need for joint optimization arises when both scaled costs $c_s$ and $c_c$ are finite and of comparable magnitude. More specifically, when $c_c < \bar c_c(c_s)$, where $\bar c_c(c_s)$ is a threshold that depends on the communication cost parameter $c_s$, the optimal solution to OPT$_{\text{Joint}}$ involves strictly positive values of both the message precision parameter $\rho^*(c_s,c_c)$ and the recommendation set exponent $\alpha^*(c_s,c_c)$. This corresponds to a hybrid interaction policy in which the user provides partial but informative communication, while the agent supplies a recommendation set whose size grows exponentially with $d$. This hybrid regime highlights that in complex environments, optimal performance requires jointly leveraging communication and search, rather than relying exclusively on either mechanism. The optimal values $\rho^*(c_s,c_c)$ and $\alpha^*(c_s,c_c)$ characterize the asymptotic scaling of the optimal message precision $\kappa$ and recommendation set size $n$, respectively. The underlying intuition for these scaling laws was developed in Section~\ref{sec:intuitive_high_dim}.

A natural illustration of the hybrid regime arises in everyday \textit{online shopping} for differentiated products such as electronics, apparel, or home goods, on e-commerce websites such as Amazon. In these settings, customers often find it cognitively costly to articulate all relevant preferences (e.g., color, brand affinity, or design features) that jointly determine the best product fit. At the same time, the sheer scale of available options renders exhaustive search impractical. Consequently, effective interaction requires combining both mechanisms: users convey coarse but informative preference signals, while the agent presents a carefully sized set of recommendations that spans the remaining uncertainty, leading to a hybrid policy.

\paragraph{Search-Only Regime
} When communication cost exceeds a certain threshold $\bar c_c(c_s)$, the optimal solution is a pure search policy. In this regime, the optimal message precision collapses to $\rho^*(c_s,c_c)=0$, and the user refrains from providing any informative communication. In this case, the user’s payoff is derived entirely from search. The user's decision to not provide preference information in the search-only regime is not due to a lack of informative potential. Under the same cost structure, the user might optimally provide partial information if the agent were restricted to a single recommendation ($n=1$). However, due to high communication cost, the ability to search cheaply renders costly communication economically inefficient.

A representative example of this regime is \textit{window shopping}, in which users face high cognitive costs in articulating precise preferences, as they themselves may be uncertain about what they seek. This corresponds to a setting with high communication costs, as transmitting informative signals is either infeasible or too cognitively demanding. Consequently, the informational burden shifts entirely to search, and the agent must compensate by offering a large recommendation set to cover the space of plausible user preferences. This interaction closely resembles a brute-force browsing experience on a generic e-commerce platform, where identifying a suitable product relies almost exclusively on the user’s capacity to search through many alternatives. We defer further discussion on the phase transition between the two regimes and the switching curve $\bar c_c(c_s)$ to Section \ref{sec:post_tilt_comparision}.

\paragraph{Frictionless Extremes}
Apart from the above mentioned regimes, there exist additional extremes outside the $1/d$ cost scaling. If either $\lambda_s \ll 1/d$ or $\lambda_c \ll 1/d$, which is analogous to $c_s=0$ or $c_c=0$, respectively, the interaction becomes effectively frictionless. In the former case, the user can search at negligible cost; in the latter, the user can communicate preferences with essentially perfect precision. Under either scenario, frictions vanish asymptotically, and the resulting expected payoff converges to the maximal attainable value $1$.

\paragraph{Absence of Communication-Only Regime} Under the high-dimensional approximation, a communication-only policy corresponds to choosing a singleton recommendation set, i.e., $\alpha = 0$, so that the user relies entirely on the message and performs no search. Unlike the search-only regime identified in   Theorem~\ref{thm:asym_opt}(3), this outcome does not arise over any nondegenerate range of parameters. In particular, $\alpha = 0$ is optimal only in two boundary scenarios. First, when $c_s \rightarrow \infty$, search becomes prohibitively costly, forcing the user to rely exclusively on communication. Second, when communication is frictionless, i.e., $c_c = 0$, the user can transmit preferences without cost, eliminating the need for search. In contrast, whenever communication carries any strictly positive cost (resulting in optimal $\rho <1$), the user strictly benefits from allowing at least some exploration, i.e., $\alpha > 0$. A singleton recommendation set, generated through posterior sampling, prevents residual preference uncertainty through search, making pure communication suboptimal away from these boundary cases.
Consequently, while $\rho^*(c_s, c_c)$ exhibits a genuine threshold behavior in $c_c$, the optimal set-size exponent $\alpha^*(c_s, c_c)$ does not display an analogous phase transition; the communication-only policy appears only as a boundary solution rather than as an interior regime.

\subsubsection{Numerical results for \texorpdfstring{$d \rightarrow \infty$}{d -> infinity} case}
The plots in Figure \ref{fig:opt_solutions} present the optimal values of $\rho^*(c_s, c_c)$ (left) and $\alpha^*(c_s,c_c)$ (middle) that solve the optimization problem OPT$_{\text{Joint}}$ for given values of $c_s$ and $c_c$. The left panel of Figure \ref{fig:opt_solutions} reveals that, for any fixed value of $c_s$, there exists a threshold $\bar{c}_c(c_s)$ at which the interaction policy undergoes a sharp transition from a hybrid communication–search regime to a purely search-only regime, as shown in Theorem \ref{thm:asym_opt}. For values of $c_c < \bar{c}_c(c_s)$, the system operates in a hybrid communication-search regime, where the optimal interaction policy is to jointly optimize over the message precision and recommendation set size. It can also be observed that when $c_c$ is small relative to $c_s$, the optimal solution favors higher values of $\rho^*(c_s, c_c)$ and $\alpha^*(c_s, c_c)$ is close to zero, even though $\alpha^*(c_s, c_c)=0$ only at the boundary $c_c = 0$.
\begin{figure}[t]
    \centering
    \includegraphics[width=0.98\linewidth]{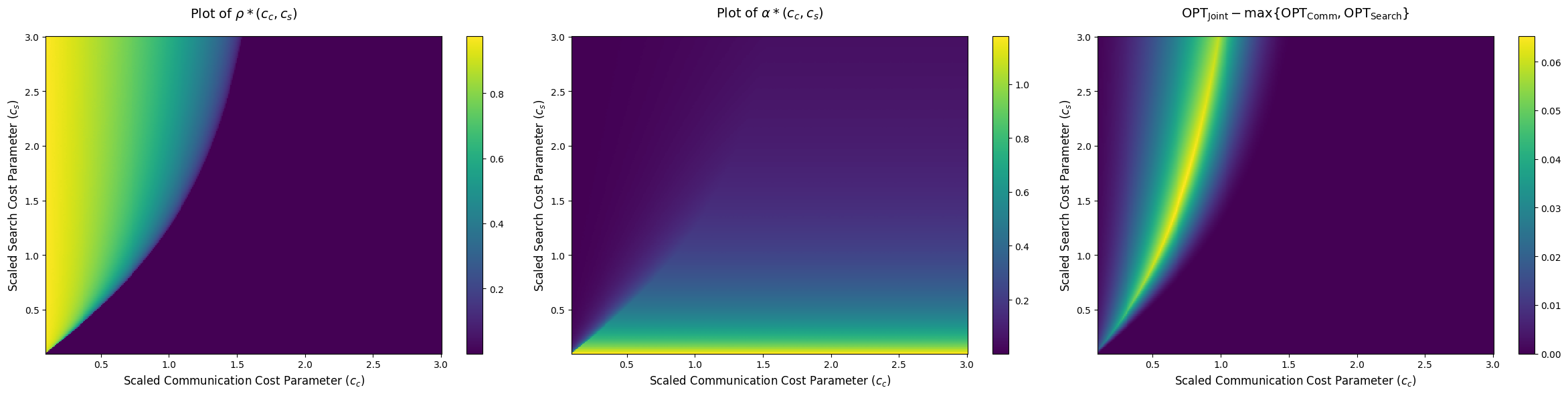}
    \setlength{\belowcaptionskip}{-14pt}
    \caption{Heatmap for $\rho^*(c_s, c_c)$ (left) and $\alpha^*(c_s,c_c)$ (middle) with respect to $c_s$ and $c_c$; and (right) incremental payoff of the joint policy relative to the best pure policy.}
    \label{fig:opt_solutions}
    \Description{The provided figure illustrates the optimal user-agent interaction policy in a high-dimensional feature space. The left and middle panels demonstrate a hybrid regime where message precision and recommendation set size are jointly optimized to navigate communication and search costs. A sharp phase transition (the curved boundary) reveals that when communication costs become sufficiently high relative to search costs, the user optimally switches to a search-only strategy, setting message precision to zero. The right panel highlights that the performance gains from this joint optimization (compared to pure search or pure communication) are most significant in intermediate cost regimes.}
\end{figure}
The right panel in Figure \ref{fig:opt_solutions} presents the heatmap of $\big[\text{OPT}_{\text{Joint}} - \max \{\text{OPT}_{\text{Comm}}, \text{OPT}_{\text{Search}}\}\big]$, where $\text{OPT}_{\text{Comm}}$ and $\text{OPT}_{\text{Search}}$ are obtained by setting $\alpha = 0$ and $\rho = 0$ in the optimization problem for $\text{OPT}_{\text{Joint}}$, respectively, and serve as asymptotic characterizations of the pure communication and pure search policies, respectively. As such, the right panel in Figure \ref{fig:opt_solutions} presents the incremental gain of using joint optimization compared to the better of the two pure benchmarks. The results show that performance gains are most significant in a narrow band of intermediate cost regimes, where $c_c$ is smaller than but comparable to $c_s$. In these settings, jointly optimizing message precision and recommendation set size yields meaningful improvements over either pure strategy, whereas outside this region the incremental benefit of joint optimization is limited.

\subsection{Theoretical Results and Insights for Finite \texorpdfstring{$d$}{d}}
\label{sec:results_finite_d}
We now turn to the finite-dimensional problem and assess the accuracy of the asymptotic prescriptions derived in the previous section. The next theorem presents a performance gap between the asymptotically optimal policy, and that under finite $d$.

\begin{theorem}[Performance Gap]
\label{thm: perf_gap}
Suppose $c_c >0$ and $c_s>0$. Let $(\rho^*(c_s,c_c), \alpha^*(c_s,c_c))$ denote the optimal solution to $\mathrm{OPT}_{\mathrm{Joint}}$ with scaled cost parameters $c_s = d\lambda_s$ and $c_c = d\lambda_c$. Then, the performance gap between the true optimal policy $(\kappa^*_d, n^*_d)$ and the asymptotically optimal policy $(\kappa^*_\infty, n^*_\infty)$ (see Eq. \eqref{eq:high_dim_interaction_policy}) satisfies
\begin{align}
    0 \ \leq \ \mathcal{P}_d \big(\kappa^*_d, n^*_d\big) - \mathcal{P}_d\big(\kappa^*_\infty, n^*_\infty\big) \ \leq \ K_{\ref{thm: perf_gap}} \sqrt{\frac{\log d}{d}},
\end{align}
where the constant $K_{\ref{thm: perf_gap}}$ depends on the cost parameters $c_c$ and $c_s$.
\end{theorem}

Theorem \ref{thm: perf_gap} establishes the accuracy of the high-dimensional approximation by quantifying the performance gap, which decays at a rate of $\tilde{O}(d^{-1/2})$. The proof of Theorem \ref{thm: perf_gap} is provided in Appendix \ref{app:proof_thm_perf_gap}. A crucial distinction from Theorem \ref{thm:asym_opt} is that this result does not rely on asymptotic scaling assumptions for the cost parameters. Instead, for any finite $d$, the theorem effectively defines the scaled parameters $c_s = d\lambda_s$ and $c_c = d\lambda_c$ and solves $\mathrm{OPT}_{\mathrm{Joint}}$ to derive the optimal $\rho^*(c_s,c_c)$ and $\alpha^*(c_s,c_c)$, which are then mapped to the control parameters $\kappa^*_\infty$ and $n^*_\infty$. From a practical standpoint, this implies that system designers can rely on the simpler asymptotic optimization to select message precision and recommendation set size without incurring significant utility loss, particularly in modern environments where embeddings are inherently high-dimensional.

\subsubsection{Discretization and the Emergence of Communication-Only Regime}
\label{sec:finite_d_regimes}
A key distinction between the finite-dimensional setting and its high-dimensional approximation lies in the granularity of the recommendation set size. In the asymptotic limit ($d \to \infty$), the recommendation set size is governed by a continuous exponent $\alpha$, allowing for smooth transitions between regimes. However, when $d$ is finite, $n$ is strictly constrained to the set of positive integers ($n \in \{1, 2, \dots\}$). This introduces a ``discretization gap'', particularly in the transition from a single recommendation ($n=1$) to multiple ($n \ge 2$). The marginal search cost of adding a second item is significant ($\lambda_s \log 2$), creating a barrier that prevents the smooth adoption of hybrid policy. Consequently, the system is more prone to ``stick'' to pure policies at the extremes of the cost spectrum, as formalized below.

\begin{proposition}[Phase Transitions in Finite-Dimensions]
\label{prop:finite_phase_transition}
    Fix the dimension $d$. We have,
    \begin{enumerate}
        \item \emph{Pure Communication regime:} For all $\lambda_s > 0$, there exists $\underline \lambda_c(\lambda_s)$, such that for all $\lambda_c < \underline \lambda_c(\lambda_s)$ we have  $n^*_d = 1$.
        \item \emph{Pure Search regime:} For all  $\lambda_s > 0$, there exists $\bar \lambda_c(\lambda_s)$, such that for all $\lambda_c > \bar{\lambda}_c(\lambda_s)$ we have $\kappa^*_d = 0$.
    \end{enumerate}
\end{proposition}

The result in Proposition \ref{prop:finite_phase_transition} highlights that while the joint  policy is dominant in the hybrid regime, the integer constraints on the recommendation set size give rise to a \textit{communication-only} regime. Specifically, when the communication cost parameter is small ($\lambda_c < \underline \lambda_c(\lambda_s)$), the discrete jump in search cost required to move from $n=1$ to $n=2$ outweighs the marginal benefit of multiple recommendations, locking the system into a \textit{pure communication} policy. Conversely, when communication is expensive ($\lambda_c > \bar{\lambda}_c(\lambda_s)$), the system enters a \textit{search-only} regime, mirroring the asymptotic result in Theorem \ref{thm:asym_opt}.

A natural example of the communication-only regime is the \textit{deep research} setting, where the desired output is a long, detailed report or analysis. Such outputs are cognitively costly for the user to evaluate, making it infeasible to provide or review multiple alternatives. As a result, the interaction effectively collapses to a single recommendation and the overall performance is highly sensitive to the quality of the user’s initial message, shifting the burden entirely to communication. Anticipating that only one output will be produced, it is worthwhile for the users to invest effort in providing a high-fidelity, well-specified query so that the agent can tailor a single precise response that closely aligns with the user's preferences.

\subsubsection{Empirical Results for Finite \texorpdfstring{$d$}{d}}
In this section, we present empirical results for the solution of $\mathcal{P}_{d}(\cdot,\cdot)$ for finite $d$. Figure~\ref{fig:finite_d_results} illustrates sharp phase transitions in the optimal interaction policy for finite dimensions, as reflected in the optimal values of message precision $\rho$ and recommendation exponent $\alpha$ across different feature dimensions $d$. The figure shows how the policy varies with the scaled communication cost $c_c$ (with $\lambda_c=c_c/d$). As $c_c$ increases, the optimal $\rho$ decreases and eventually reaches zero at a dimension-dependent threshold. By contrast, the optimal $\alpha$ increases with $c_c$, rising from zero to a strictly positive value. Notably, there are parameter regions in which both $\rho$ and $\alpha$ are simultaneously nonzero. This pattern is consistent with Proposition~\ref{prop:finite_phase_transition} and provides a clear illustration of the finite-dimensional phase transition in $\rho$. Moreover, the discrete jumps in $\alpha$ highlight the role of the integer constraint on the recommendation set size, which induces sharp regime switches that are absent in the continuous asymptotic approximation.
\begin{figure}[t]
\centering
\includegraphics[width=0.8\linewidth]{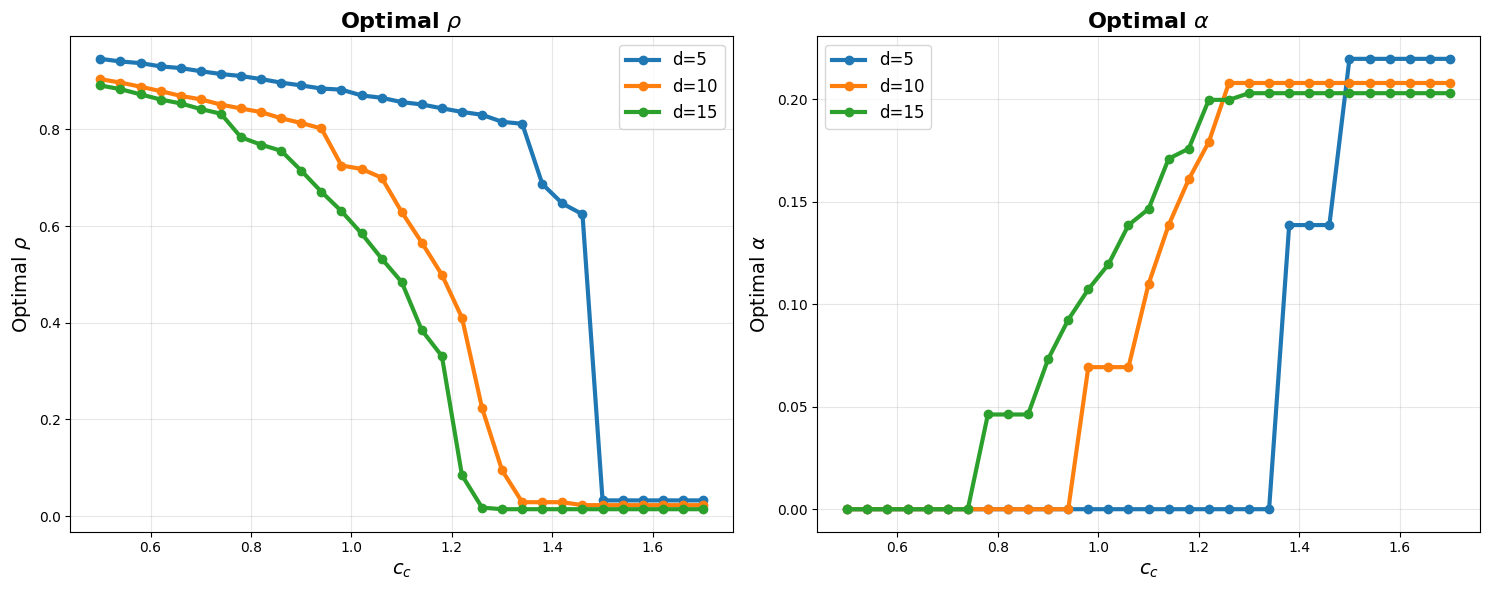}
\setlength{\belowcaptionskip}{-12pt}
\caption{Optimal control variables in finite dimensions with respect to the scaled communication cost parameter. (Results are estimated via Monte Carlo simulations and therefore exhibit sampling variability.)}

\label{fig:finite_d_results}
\end{figure}
\section{Recommendations from Tilted Distribution}
\label{sec:tilted_dist}
In previous sections, our analysis has focused on recommendation sets generated through direct sampling from the posterior distribution. While natural and easy to implement leveraging Gen AI, this approach does not take advantage of the possibility of systematically adjusting the weight placed on the user’s message. In this section, we consider an alternative mechanism. Given a message $\mathbf{m} \in [-1,1]^d$, we define the class $\mathcal{Q}$ as the collection of distributions over $\boldsymbol{\theta}$ induced by the following stochastic construction: a scalar random variable $V$ is drawn from some distribution $P \in \mathcal{P}([-1,1])$, and independently, a random vector $\mathbf{Y}$ is drawn uniformly from among unit vectors orthogonal to $\mathbf{m}$, denoted by $\mathcal{S}(\mathbf{m})$. The recommendation vector is then generated as $\boldsymbol{\theta} = V \mathbf{m} + \sqrt{1 - V^2}\, \mathbf{Y}$. Formally,
\begin{align*}
\mathcal{Q} := \left\{ q(\boldsymbol{\theta}|\mathbf{m}) \mid \boldsymbol{\theta} = V\mathbf{m} + \sqrt{1-V^2} \mathbf{Y}, \quad V \sim \mathcal{P}([-1,1]), \quad \mathbf{Y} \sim \text{Unif}(\mathcal{S}(\mathbf{m})), \quad \mathbf{Y} \perp V \right\}
\end{align*}
We define the class of \textit{tilted distributions}, in which the weight $V$ is deterministic. That is, for a fixed $v \in [-1,1]$, the recommendation is given by $\boldsymbol{\theta} = v \mathbf{m} + \sqrt{1 - v^2}\, \mathbf{Y}$, with $\mathbf{Y} \sim \mathrm{Unif}(\mathcal{S}(\mathbf{m}))$. Formally,
\begin{align*}
    \mathcal{Q}^{\text{Tilt}} := \left\{ q(\boldsymbol{\theta}|\mathbf{m}) \mid \boldsymbol{\theta} = v\mathbf{m} + \sqrt{1-v^2} \mathbf{Y}, \quad v \in [-1, 1], \quad \mathbf{Y} \sim \text{Unif}(\mathcal{S}(\mathbf{m})) \right\}.
\end{align*}

The following theorem compares the performance of the general class $\mathcal{Q}$ with that of the tilted subclass $\mathcal{Q}^{\mathrm{Tilt}}$. It shows that, in high dimensions, randomizing the weight placed on the user’s message provides no asymptotic advantage over an optimally chosen deterministic tilt.

\begin{theorem}[Optimality gap of sampling from a tilted distribution]
\label{thm:tilt_opt}
For a recommendation set of size $n$, the expected utility of a sampling scheme $q \in \mathcal{Q}$ is given by
\begin{align*}
U_n(q) := \mathbb{E}_{\boldsymbol\theta_1, \dots, \boldsymbol\theta_n \sim q(\cdot | \mathbf{m}), \text{ i.i.d.}} \left[ \max_{i \in [n]} \langle \mathbf{h}, \boldsymbol\theta_i \rangle \right].
\end{align*}
Then, there exists a universal constant $K_{\ref{thm:tilt_opt}}$, such that, for $d \geq 4$, the maximum expected utility achievable by any distribution in the class $\mathcal{Q}$ satisfies
\begin{align*}
\max_{q \in \mathcal{Q}} \mathbb{E}_{\mathbf{h}, \mathbf{m}} [U_n(q)] \leq \max_{q \in \mathcal{Q}^{\text{Tilt}}} \mathbb{E}_{\mathbf{h}, \mathbf{m}} [U_n(q)] + \frac{K_{\ref{thm:tilt_opt}}}{\sqrt{d-3} (1-\rho^2)},
\end{align*}
where the expectation is taken over the joint distribution of preferences $\mathbf{h}$, and messages $\mathbf{m}\sim p_{\kappa}(\cdot|\mathbf{h})$, where $\rho \in [0,1)$ and $\kappa = \frac{\rho}{(1-\rho^2)} (d-3)$.
\end{theorem}

Theorem \ref{thm:tilt_opt} implies that the agent can deterministically \emph{tilt} the recommendations toward the observed message to achieve near-optimal expected utility, effectively managing the crucial trade-off between exploiting communicated preferences and maintaining diversity across unobserved feature dimensions. The tilt parameter $v$ thus becomes a design lever: when communication is reliable, the agent can bias recommendations more heavily toward the message; when communication is noisy, the agent can diversify recommendations to hedge against uncertainty. Mathematically, the optimal tilt is the solution of the following optimization problem:
\begin{align*}
    v^* = \arg\max_{v} \ \ v\mathbb E[W] + \sqrt{1-v^2} \mathbb E\left[\sqrt{1-W^2}\right] \mathbb E\Big[\max_{i\in[n]} X_i \Big],
\end{align*}
where $W$ and $X_i$'s satisfy the characterization in Lemma \ref{lem:prelim_results}. Note that the optimal tilt parameter $v^*$ depends on the value of $\mathbb E\left[\max_i X_i \right]$, which in turn depends on the recommendation set size $n$, implying the dependence of optimal tilt on $n$, unlike that for the posterior sampling approach. Exploiting Theorem \ref{thm:tilt_opt}, we restrict subsequent analysis to the class of tilted distributions.

Allowing the user to also choose the message precision $\kappa$, the joint objective becomes
\begin{align*}
    \mathcal{T}_d(\kappa,n,v) = \ v\mathbb E[W] + \sqrt{1-v^2} \mathbb E\left[\sqrt{1-W^2}\right] \mathbb E\left[\max_i X_i   \right] -\lambda_s \log n - \lambda_c \Big(\kappa \EE[W] - \log \frac{C_d(0)}{C_d(\kappa)}\Big).
\end{align*}
We also define the high-dimensional approximation of the objective as
\begin{align}
\label{eq:tilt_high_dim_obj}
    \mathcal{T}_{\infty} (\rho,\alpha,v) = \max_{x} \ \ \rho v + \sqrt{1-\rho^2} \sqrt{1-v^2} x - c_s \alpha + \frac{1}{2} c_c \log (1-\rho^2) \ \text{ such that } I(x) \leq \alpha,
\end{align}
where $I(x)  = -\frac{1}{2} \log(1-x^2)$, implying that optimal $x$ is given by $x^* = \sqrt{1-e^{-2\alpha}}$. The relation between $\mathcal{T}_d(\kappa,v,n)$ and the high-dimensional counterpart $\mathcal{T}_{\infty} (\rho,\alpha,v)$ is captured in Theorem \ref{thm:tilt}.

\begin{theorem} [Asymptotically optimal tilt parameter and Phase Transitions]
\label{thm:tilt}
Suppose $\rho$ and $\alpha$ are fixed, and let, for any $d$, $\kappa_d = \frac{\rho}{1-\rho^2}(d-3)$ and $n_d = \lfloor e^{\alpha d} \rfloor$ for fixed $\rho$ and $\alpha$. Then, for any choice of tilt parameter $v$, we have
\begin{align*}
    \lim_{d\rightarrow \infty} \mathcal{T}_d(\kappa_d ,n_d, v) = \mathcal{T}_{\infty} (\rho,\alpha,v).
\end{align*}
Furthermore,
\begin{enumerate}
    \item \emph{Optimal Tilt:} For any fixed $(\rho, \alpha)$, the optimal tilt parameter $v^*$ that optimizes the high-dimensional approximation of the objective, i.e., $\mathcal{T}_{\infty} (\rho,\alpha,v)$, is given
    by
    \begin{equation}
    \label{eq:opt_tilt}
        v^*(\rho, \alpha) = \frac{\rho}{\sqrt{1 - (1-\rho^2) e^{-2\alpha}}}.
    \end{equation}
    \item \emph{Regime Phase Transitions:} Let $(\rho^*, \alpha^*)$ maximize the asymptotic utility $\mathcal{T}_{\infty}$ subject to search cost $c_s$ and communication cost $c_c$. The optimal policy exhibits a sharp phase transition:
\begin{itemize}
    \item \emph{Pure Communication Regime:} If $c_s > c_c$, the optimal policy collapses to $\alpha^* = 0$ and $v^* = 1$.
    \item \emph{Pure Search Regime:} If $c_s < c_c$, the optimal policy collapses to $\rho^* = 0$ and $v^* = 0$.
\end{itemize}
Further, we have that
\begin{align}
\label{eq:rho_alpha_under_tilt}
   \big(1-(\rho^*)^2\big)e^{-2\alpha^*} =
   \frac{1}{2}  \Big( \sqrt{c_{\min}^4+ 4 c_{\min}^2} -c_{\min}^2 \Big),
\end{align}
where $c_{\min} = \min\{c_s , c_c\}$, and  either $\rho^* = 0$ (if $c_c > c_s$) or $\alpha^* = 0$ (if $c_c < c_s$).
\end{enumerate}
\end{theorem}

Theorem~\ref{thm:tilt} characterizes the optimal tilt $v^*$ under the high-dimensional approximation of the objective $\mathcal{T}_d(\kappa, n, v)$ and identifies the sharp phase transition that emerges when both the tilt $v$ and the parameters $(\rho, \alpha)$ are chosen optimally. We see that the optimal tilt $v^*$ depends on both $\rho$ and $\alpha$, i.e., the optimal tilt must adjust simultaneously to the message precision and to the size of the recommendation set. Several key structural observations emerge from Theorem~\ref{thm:tilt}. First, whenever $\rho>0$ and $\alpha < \infty$ (i.e., $n < \infty$), the optimal tilt $v^*$ is strictly greater than $\rho$. This indicates that an optimized agent should place more weight on the communicated signal than the typical weight under  posterior sampling, which can be seen as a form of regularization. Second, the availability of tilting induces a specialization in the optimal policy: the policy either extracts information and uses it to provide a single recommendation ($\alpha^*=0$ and $v^* = 1$) or avoids communication entirely to rely on search ($\rho^* = 0$ and $v^*=0$). \emph{There is no joint communication and search regime.} And third, as can be deduced from Eq. \eqref{eq:rho_alpha_under_tilt}, the optimal precision $\rho^*$ and set-size exponent $\alpha^*$ are determined by the minimum of the two cost parameters $c_{\min}$, exhibiting  monotonicity where $\rho^*$ decreases as $c_c$ increases and $\alpha^*$ decreases as $c_s$ increases.

We avoid pursuing a more elaborate convergence result (cf. Theorem \ref{thm:asym_opt}) and performance gap bound (cf. Theorem \ref{thm: perf_gap}) in the interest of space. The proofs of Theorem \ref{thm:tilt_opt} and Theorem~\ref{thm:tilt} are provided in Appendix~\ref{app:proofs_tilt}. Figure~\ref{fig:tilt} (left) compares the payoff achieved under tilting with that obtained from posterior sampling; and reveals substantial gains from tilting under $c_s> c_c$.

\begin{figure}[ht!]
    \centering
    \includegraphics[width=\linewidth]{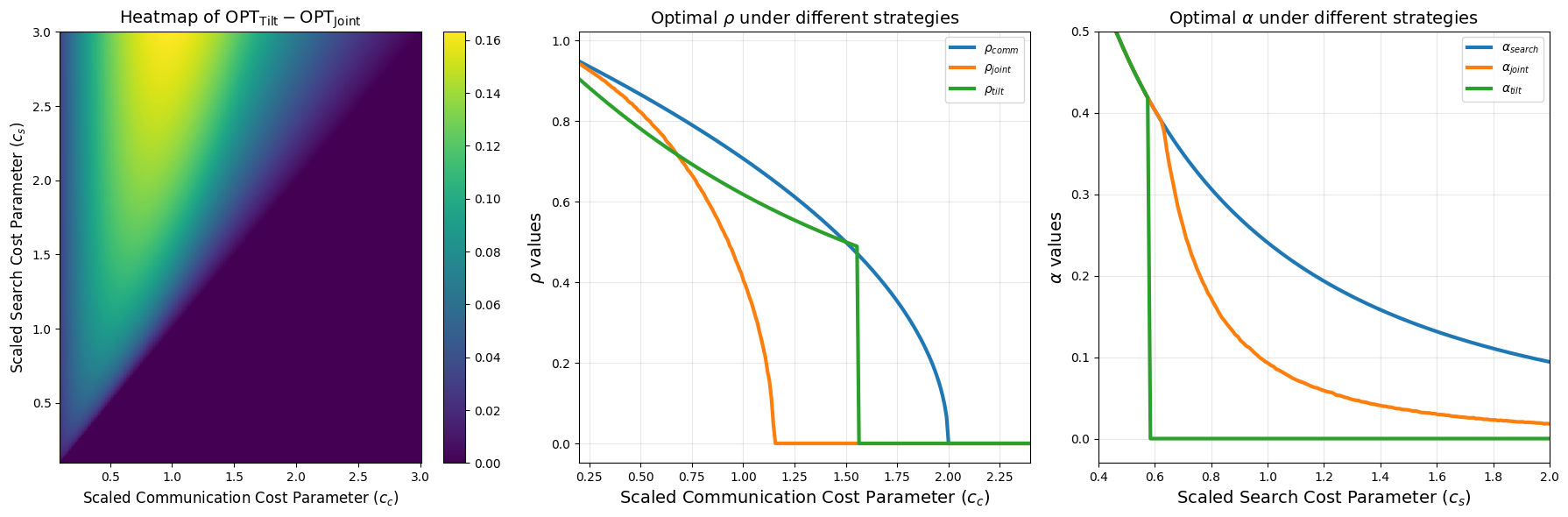}
    \setlength{\belowcaptionskip}{-8pt}
    \caption{Performance comparison between the tilted-distribution and posterior-sampling strategies, along with the comparison of $\rho^*$ w.r.t. $c_c$ (with $c_s = 1.55$)  and $\alpha^*$ w.r.t. $c_s$ (with $c_c = 0.58$).}

    \label{fig:tilt}
\end{figure}

Practically, the tilt parameter corresponds to a tunable control in modern LLM-based recommender systems, analogous to temperature or retrieval-weighting in retrieval-augmented generation (RAG) pipelines \citep{lewis2020retrieval, gao2023retrieval}.  Crucially, the structure of the optimal tilt parameter reveals that the real-world tuning must account for both the quality of the user's message and the desired diversity in the output. This mirrors emerging practices in commercial AI assistants, which increasingly adjust weighting or temperature dynamically based on query quality and uncertainty \citep{wu2024survey, bao2024large}.

\subsubsection{Comparison of Posterior Sampling and Tilted Distribution Schemes}
\label{sec:post_tilt_comparision}
This section compares tilted sampling and posterior sampling, and the optimal interaction policy under each.

\paragraph{Structure of optimal interaction policy:} A striking implication of Theorem \ref{thm:tilt} is that, when all the parameters $(\rho,n,v)$ are chosen optimally, the interaction policy is governed by the smaller of the two cost parameters. Unlike the posterior-sampling approach, where we have to jointly optimize communication and search, the tilted mechanism adapts primarily to the cheapest source of user effort. When communication is cheaper, the system relies on precise messaging; when search is cheaper, it relies on exploration. Further, as seen in Figure \ref{fig:tilt}, the optimal message precision under tilting declines more gradually as communication costs increase, indicating that the users have reason to provide more preference information under tilting.
Also, the tilted sampling generally relies on smaller recommendation sets, with the optimal recommendation set size exponent consistently no larger than that under posterior sampling. Together, these patterns show that tilting makes more effective use of both communication and search than direct posterior sampling.

\paragraph{Switching curves:} Under posterior sampling, the optimal policy switches from a hybrid policy to a search-only regime at a threshold on $c_c$ of $\bar c_c(c_s) < c_s$, whereas under tilted sampling, there is sharp phase transition at a threshold of $c_c=c_s$. The shift from the asymmetric switching curve under posterior sampling to the symmetric $c_c = c_s$ line under tilting results from a move from a passive, ``black-box'' agent to a highly optimized one. Under posterior sampling, the agent's recommendations merely mirror the user's noisy input, which causes the marginal benefit of communication to diminish rapidly as the user’s message becomes coarser (as communication cost increases). This makes search a cheaper lever for utility, leading users to abandon communication even when it is technically cheaper than search, leading to $\bar{c}_c(c_s) < c_s$.

In contrast, tilted sampling allows the agent to act as an active optimizer that makes best use of the user's signal by adjusting a deterministic tilt parameter. This capability restores the efficiency of user's communication; specifically, for $\bar c_c(c_s) < c_c < c_s$, tilting facilitates a shift where the user’s effort is reallocated from search (under posterior sampling) back to communication (under tilting). This effectively restores the marginal parity between communication and search, causing the switching curve to coincide exactly with $c_c = c_s$.

\section{Conclusion}
This work establishes a unified theoretical framework for understanding the collaborative interplay between human communication and AI-driven recommendation. By viewing this interaction through an information-theoretically inspired lens, we capture how users’ cognitive costs for communication and search shape optimal system design and overall performance. Our findings may help inform the design of efficient AI shopping assistants which use the right mix of communication and search for a given situation, and if possible, are tuned to sample recommendations from an appropriately ``tilted'' version of the distribution of what the user may like.

\bibliographystyle{plainnat}
\bibliography{references}

\appendix
\section{Pure Policies under Posterior Sampling}
\label{sec:pure_strategies}

To highlight the benefits of jointly optimizing message precision and recommendation set size under posterior sampling, we compare our main formulation with two benchmark policies in which only one lever is active. Specifically, we consider a pure search benchmark, where only the recommendation set size is optimized, and a pure communication benchmark, where only message precision is optimized. These restricted scenarios isolate the individual contributions of search and communication and allow us to quantify the performance gains from optimizing both levers simultaneously.

\subsection{Brute-Force Search (No Communication)} In this scenario, the user provides no preference information to the agent, which corresponds to setting the message precision $\kappa=0$. Mathematically, the high-dimensional approximation of the user's payoff, $\mathcal{P}_d(0,n)$, is found by substituting $\rho = w = 0$ into the joint optimization problem $\text{OPT}_{\text{Joint}}$, presented in Eq.~\eqref{eq:joint_opt}. The problem then simplifies to,
\begin{align*}
    \text{OPT}_{\text{Search}} = \max_{\alpha, x} \left\{ x - c_s \alpha \ \ \ \text{s.t.} \ \ \ \frac{1}{2} \log(1-x^2) = - \alpha \right\}.
\end{align*}
This constrained optimization problem can be solved analytically, yielding the optimal values
\begin{align}
    \alpha^*(c_s) := \frac{1}{2} \log \left( \frac{1}{2} + \sqrt{\frac{1}{4} + \frac{1}{c_s^2}} \right), \quad \text{and} \quad \text{OPT}_{\text{Search}} = \sqrt{1-e^{-2\alpha^*(c_s)}} - c_s \alpha^*(c_s).
\end{align}
From Theorem \ref{thm:asym_opt}, we also know that there exists $\bar c_c$ such that for all $c_c > \bar c_c$, we have that the optimal solution from $\text{OPT}_{\text{Joint}}$ matches that of $\text{OPT}_{\text{Search}}$. The following corollary provides the performance gap between the optimal performance of the the high-dimensional approximation and that of the original finite-dimensional problem.
\begin{corollary}
\label{cor:brute_force_gap}
Let $\alpha^*(c_s)$ denote the solution of $\text{OPT}_{\text{Search}}$ and let $n^*_\infty = \lfloor e^{d \alpha^*(c_s)} \rfloor$. We have,
\begin{align*}
    \left| \max_{n} \mathcal{P}_d(0,n) - \mathcal{P}_d(0,n^*_\infty) \right| \leq \tilde{O}\big(d^{-\frac{1}{2}}\big).
\end{align*}
\end{corollary}
Mathematically, Corollary~\ref{cor:brute_force_gap} follows by arguments analogous to those used in the proof of Theorem~\ref{thm: perf_gap}. As shown in Proposition~\ref{prop:finite_phase_transition}, the pure search policy is not merely a theoretical baseline but becomes the \emph{optimal} interaction policy when communication costs are prohibitively high. In this regime, the preceding analysis yields a closed-form characterization of the optimal recommendation set size.

The analytic expression for $\alpha^(c_s)$ also allows us to characterize system performance under high search costs. A second-order expansion gives $\alpha^(c_s) \approx \frac{1}{2c_s^2}$, which implies that the optimal payoff satisfies $\text{OPT}_{\text{Search}} \approx \frac{1}{2c_s}$. Thus, the attainable payoff decays on the order of $1/c_s$, indicating a severe deterioration in performance as search becomes more costly for the user.

\subsection{Single recommendation with quality communication} In this scenario, the agent is restricted to providing a single recommendation ($n=1$). This is analogous to setting $\alpha = 0$ in the high-dimensional approximation, and the feasibility constraint in the $\text{OPT}_{\text{Joint}}$ problem, $I_{\rho}(w,x) = \alpha$, simplifies to $I_{\rho}(w,x)=0$. This equality holds only when the recommendation perfectly aligns with the communicated message ($w=\rho$), and there is no orthogonal component to explore ($x=0$). The optimization problem thus reduces to
\begin{align*}
    \text{OPT}_{\text{Comm}} = \max_{\rho} \left\{\rho^2 + \frac{c_c}{2} \log (1-\rho^2)\right\}.
\end{align*}
The optimal solution exhibits a (soft) threshold depending on $c_c$, given by
\begin{align*}
\rho^*(c_c) :=
\begin{cases}
    \sqrt{1-\tfrac{c_c}{2}}, & \text{if } c_c < 2, \\[6pt]
    0, & \text{otherwise},
\end{cases}
\qquad \text{and} \qquad
\text{OPT}_{\text{Comm}} = \begin{cases}
    1 - \tfrac{c_c}{2}\log \tfrac{2e}{c_c}, & \text{if } c_c < 2, \\[6pt] 0, & \text{otherwise}.
\end{cases}
\end{align*}
This result reveals that meaningful communication is only viable if the scaled cost $c_c$ is below a threshold of $2$. If the cost exceeds this point, the optimal policy for the user is to provide no information ($\rho^*(c_c)=0$), resulting in a payoff of zero, which is equivalent to the agent making a random guess. Note that, for any value of $c_c$, the result of $\text{OPT}_{\text{Joint}}$ matches that of $\text{OPT}_{\text{Comm}}$ only when $c_s$ goes to infinity. As before, the following corollary ensures that this simplified model accurately reflects the behavior of the finite-dimensional system.

\begin{corollary}
\label{cor:comm_gap}
Let $\rho^*(c_c)$ and $\rho^*(c_s,c_c)$ denote the solution of $\text{OPT}_{\text{Comm}}$ and $\text{OPT}_{\text{Joint}}$ respectively. Then, for any $c_s$, $\rho^*(c_c) = \rho^*(\infty,c_c) \geq \rho^*(c_s,c_c)$. Let $\kappa^*_\infty = \frac{\rho^*(c_c)}{1-(\rho^*(c_c))^2}d$. We have,
\begin{align*}
    \big| \max_{\kappa} \mathcal{P}_d(\kappa,1) - \mathcal{P}_d(\kappa^*_\infty,1) \big| \leq \tilde{O}\big(d^{-\frac{1}{2}}\big).
\end{align*}
\end{corollary}
Similar to Corollary~\ref{cor:brute_force_gap}, Corollary~\ref{cor:comm_gap} follows by arguments analogous to those used in the proof of Theorem~\ref{thm: perf_gap}. Corollary~\ref{cor:comm_gap} shows that under a pure communication regime, the user is incentivized to invest more effort in providing preference information. Moreover, the user’s achievable payoff declines rapidly with the communication cost parameter, eventually collapsing to zero at a finite threshold.

The preceding analysis highlights the inherent fragility of relying on pure policies. In both the search-only and communication-only benchmarks, the user’s attainable payoff declines rapidly as the corresponding cost parameters increase, rendering these single-lever approaches ineffective even under moderate cognitive burdens. This limitation reinforces our central insight: effective system design requires jointly optimizing message precision and recommendation set size in accordance with the underlying cost structure.

\section{Additional Numerical Results}

To validate the accuracy of our analytical approach, Figure~\ref{fig:finite_d} compares the optimal objective values and solutions of the finite-dimensional problem against high-dimensional asymptotic approximations, which are evaluated through simulation. The leftmost panel reports the overall payoff (objective value) and shows that the user’s payoff (dotted) converges to the theoretical limit (dashed) as the feature dimension $d$ increases, consistent with the convergence predicted by Theorem~\ref{thm: perf_gap}. The center and right panels depict the behavior of the optimal control variables. We observe that both the message precision parameter $\kappa$ and the recommendation set size $n$ are increasing functions of the feature dimension $d$. Specifically, $\kappa$ scales linearly with $d$, while $n$ scales exponentially with $d$.

\begin{figure}[!ht]
\centering
\includegraphics[width=\linewidth]{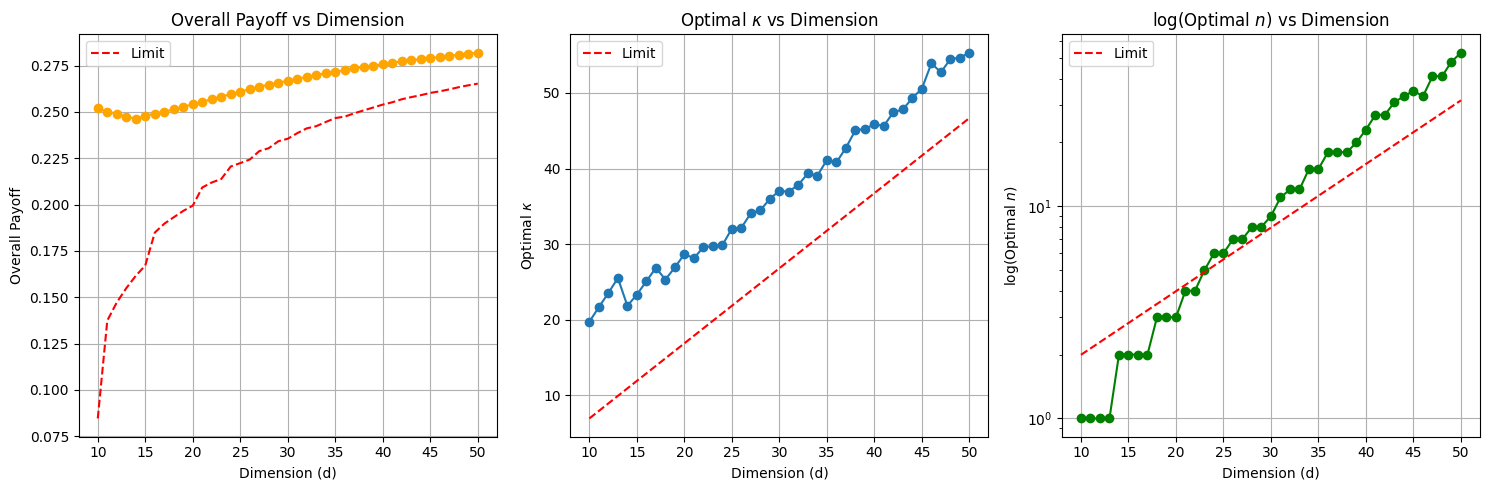}
\caption{Comparison of finite-$d$ simulations (markers) and high-dimensional asymptotics (dashed lines). Panels display the optimal expected payoff (Left), message precision $\rho$ (Center), and set size exponent $\alpha$ (Right) as a function of $d$.}

\label{fig:finite_d}
\end{figure}

\section{Weighted Preferences Extension}
\label{sec:weighted_prefer}
In the baseline model, user preferences were assumed to be uniformly distributed across feature dimensions, so that each attribute contributes equally to overall utility and incurs the same communication cost. In many practical recommendation settings, however, feature dimensions may differ both in how much they contribute to utility and in how costly it is to communicate preferences about them. In this section, we extend the framework to capture such structured heterogeneity by allowing different subsets of features to receive different weights in the utility function and to be associated with different communication costs. This formulation enables us to study how communication and search decisions are allocated across feature subsets when their contributions to utility and their communication costs are asymmetric.

We model the user’s preference vector as a composition of two orthogonal components:
\begin{align*}
    \mathbf{h} = \mu \mathbf{h}_1 + \sqrt{1-\mu^2} \mathbf{h}_2
\end{align*}
where, for simplicity, we assume that the set of features are divided into two subsets $d= d_1+d_2$, and $\mathbf{h}_{1}$ is supported on the first $d_1$ features (last $d_2$ elements of $\mathbf{h}_{1}$ are all zero), and similarly, $\mathbf{h}_{2}$ is supported on the last $d_2$ features. The scalar $\mu \in [0,1]$ determines the relative importance of the two feature groups.

The user’s message follows a similar decomposition, $\mathbf{m} = \mu \mathbf{m}_1 + \sqrt{1-\mu^2} \mathbf{m}_2$, where $\mathbf{m}_1$ and $\mathbf{m}_2$ are generated according to vMF distributions on $\mathcal{S}^{d_1-1}$ and $\mathcal{S}^{d_2-1}$ with precision parameters $\kappa_1$ and $\kappa_2$, respectively. To allow for differential ease of communication across feature groups, we also permit group-specific communication costs, denoted by $\lambda_{1,c}$ and $\lambda_{2,c}$.

The agent, upon receiving the message, constructs recommendation sets $\{\boldsymbol{\theta}_{1,1}, \dots,  \boldsymbol{\theta}_{1,n_1}\}$ and $\{\boldsymbol{\theta}_{2,1}, \dots,  \boldsymbol{\theta}_{2,n_2}\}$ over the two feature subsets, combining them multiplicatively to form the full recommendation menu. Thus, the total number of recommendations is given by $n = n_1 \times n_2$, where each product feature vector takes the form $\boldsymbol{\theta}_{ij} = \mu \boldsymbol{\theta}_{1,i} + \sqrt{1-\mu^2} \boldsymbol{\theta}_{2,j}$. The utility that the user receives from recommendation $\boldsymbol{\theta}_{ij}$ is given by
\begin{align*}
    u(\mathbf{h} , \boldsymbol{\theta}_{ij}) = \mu^2\langle \mathbf{h}_1 , \boldsymbol{\theta}_{1,i} \rangle + (1-\mu^2) \langle \mathbf{h}_2 , \boldsymbol{\theta}_{2,j} \rangle.
\end{align*}

Under this modeling, the overall optimization problem decomposes into two independent subproblems, one for each feature subset, with search and communication costs effectively scaled by their respective utility weights. The overall payoff of the user is given by
\begin{align*}
    \mathcal{N}_{d_1,d_2} & (\kappa_1,\kappa_2 , n_1,n_2) := \mu^2 \mathcal{P}_{d_1}(\kappa_1, n_1) + (1-\mu^2)\mathcal{P}_{d_2}(\kappa_2, n_2),
\end{align*}
where each term $\mathcal{P}_{d_1}(\kappa_1, n_1)$ follows the same structure as in the baseline model, with costs $\big(\frac{\lambda_s}{\mu^2}, \frac{\lambda_{1,c}}{\mu^2} \big)$, and similarly for $\mathcal{P}_{d_2}(\kappa_2, n_2)$ with cost parameters $\big(\frac{\lambda_s}{1-\mu^2}, \frac{\lambda_{2,c}}{1-\mu^2} \big)$.

This formulation highlights the interaction between communication and search across heterogeneous feature groups. For instance, when $\lambda_{2,c} \gg \lambda_s \gg \lambda_{1,c}$, the user rationally concentrates communication effort on the more easily articulated feature subset (the first $d_1$ dimensions), while relying on search to explore the harder-to-communicate features ($d_2$ dimensions). The agent, in turn, optimally sets $n_1 = O(1)$ to avoid redundancy in the well-communicated subspace while diversifying over the poorly communicated subspace by choosing $n_2 = O(e^{\alpha d_2})$ for some $\alpha>0$. This extension provides a natural and intuitive framework for modeling asymmetric feature importance: communication effort is directed toward salient, describable attributes, while search diversity compensates for uncertainty in latent dimensions.

\paragraph{Discussion on $\mu$} Although the optimization problem mathematically decouples into two independent subproblems, the weighting parameter $\mu$ serves as a critical lever for allocating cognitive effort across the two feature subspaces. The parameter $\mu$ influences the system through two reinforcing channels: marginal utility and effective cost. First, the weights determine each subspace’s contribution to the user’s total utility. A larger $\mu$ increases the value of alignment in the first feature subspace, thereby strengthening the incentive to optimize along those dimensions. Second, and more subtly, $\mu$ effectively rescales the cost parameters. As shown in the decomposition, the optimization for the first subspace is governed by effective cost parameters $(\lambda_{1,c}/\mu^2, \lambda_s/\mu^2)$. This creates an inverse relationship between importance and cost, as the subspace becomes more important (larger weight), the effective "price" of acquiring information and searching within that subspace decreases. Consequently, when $\mu$ is large, the system optimally invests heavily in the first subspace by demanding higher message precision and, when beneficial, greater search effort. As $\mu$ decreases, the effective costs for the first subspace rise while those for the second subspace fall, and attention shifts accordingly.

\section{Proofs of Essential Lemmas}

\subsection{Proof of Lemma \ref{lem:prelim_results}}
\label{sec:proof_lem_prelims}

\begin{proof}[Proof of Lemma \ref{lem:prelim_results}.]
We provide the proof of each part separately.
\begin{itemize}
    \item[(i)] The density function of the message fidelity $W$ and the uncommunicated component $\mathbf{Y}$ is provided in \cite[Chapter 9.3]{mardia2009directional}.
    \item[(ii)]  To prove mutual independence, it is sufficient to show that for any pair of distinct indices $i \neq j$, the random variables $X_i$ and $X_j$ are independent. Fix $\mathbf{m} \in \mathcal{S}^{d-1}$ and let $\mathcal{S}(\mathbf{m}) = \{ \mathbf{v} \in \mathcal{S}^{d-1} : \langle \mathbf{m}, \mathbf{v} \rangle = 0 \}$ be the equatorial $(d-2)$-sphere orthogonal to $\mathbf{m}$. We condition on the event $\mathbf{Y} = \mathbf{y} \in \mathcal{S}(\mathbf{m})$. Under this conditioning, the random variables $[X_i | \mathbf{Y}=\mathbf{y}] = \langle \mathbf{y}, \mathbf{Y}_i \rangle$ and $[X_j | \mathbf{Y}=\mathbf{y}] = \langle \mathbf{y}, \mathbf{Y}_j \rangle$ are functions of $\mathbf{Y}_i$ and $\mathbf{Y}_j$ respectively. Since $\mathbf{Y}_i$ and $\mathbf{Y}_j$ are i.i.d. samples from $\text{Unif}(\mathcal{S}(\mathbf{m}))$, the conditional variables $X_i$ and $X_j$ are also conditionally independent:
    \begin{align*}
    P(X_i \le a, X_j \le b | \mathbf{Y}=\mathbf{y}) = P(X_i \le a | \mathbf{Y}=\mathbf{y}) P(X_j \le b | \mathbf{Y}=\mathbf{y}).
    \end{align*}
    This gives us
    \begin{align*}
    P(X_i \le a, X_j \le b) = \int_{\mathbf{y} \in \mathcal{S}^{d-1}} P(X_i \le a | \mathbf{Y}=\mathbf{y}) P(X_j \le b | \mathbf{Y}=\mathbf{y}) p_d(\mathbf{y})   dy.
    \end{align*}
    Next, we invoke rotational invariance. Let $\mathcal{O}(\mathcal{S}(\mathbf{m}))$ be the group of orthogonal transformations (rotations and reflections) that map the subspace $\text{span}(\mathcal{S}(\mathbf{m}))$ onto itself. For any two vectors $\mathbf{y}_1, \mathbf{y}_2 \in \mathcal{S}(\mathbf{m})$, there exists a transformation $\mathbf{R} \in \mathcal{O}(\mathcal{S}(\mathbf{m}))$ such that $\mathbf{y}_2 = \mathbf{R} \mathbf{y}_1$. Because $\mathbf{Y}_i$ is uniformly distributed over $\mathcal{S}(\mathbf{m})$, its distribution is invariant under $\mathbf{R}$ (i.e., $\mathbf{R}\mathbf{Y}_i \overset{d}{=} \mathbf{Y}_i$). Therefore:$$\langle \mathbf{y}_2, \mathbf{Y}_i \rangle = \langle \mathbf{R} \mathbf{y}_1, \mathbf{Y}_i \rangle = \langle \mathbf{y}_1, \mathbf{R}^T \mathbf{Y}_i \rangle \overset{d}{=} \langle \mathbf{y}_1, \mathbf{Y}_i \rangle$$This demonstrates that the conditional probability $P(X_i \le a | \mathbf{Y}=\mathbf{y})$ is identical for all $\mathbf{y} \in \mathcal{S}(\mathbf{m})$. Consequently, the conditional distribution is equal to the marginal distribution: $P(X_i \le a | \mathbf{Y}=\mathbf{y}) = P(X_i \le a)$. Similarly, we have $P(X_j \le b | \mathbf{Y}=\mathbf{y})  = P(X_j \le b)$. By substituting this back in the integral, we have
    \begin{align*}
    P(X_i \le a, X_j \le b) = P(X_i \le a) P(X_j \le b).
    \end{align*}
    This holds for any pair $i \neq j$, and therefore the set of random variables $\{X_1, \dots, X_n\}$ are mutually independent.

    Next, we characterize the distribution of $X_i$'s. Since the uniform distribution on $\mathcal{S}^{d-1}$ is invariant under orthogonal transformations, the distribution of the inner product $X_i = \langle \mathbf{Y}, \mathbf{Y}_i \rangle$ is independent of the orientation of the reference vector $\mathbf{m}$. Therefore, without loss of generality, we can choose $\mathbf{m} = \mathbf{e}_d = (0, 0, \dots, 0, 1)$. For this choice of $\mathbf{m}$, we can generate samples from $\text{Unif}\big(\mathcal{S}(\mathbf{e}_d)\big)$ as follows: first sample $\tilde{\mathbf{Y}} \sim \text{Unif}\big(\mathcal{S}^{d-2}\big)$, then set $\mathbf{Y} = (\tilde{\mathbf{Y}}, 0)$. Similarly, we can write $\mathbf{Y}_i = (\tilde{\mathbf{Y}}_i, 0)$, where $\tilde{\mathbf{Y}}_i$'s are independent samples from $\text{Unif}\big(\mathcal{S}^{d-2}\big)$. Consequently, $X_i = \langle \mathbf{Y}, \mathbf{Y}_i \rangle = \langle \tilde{\mathbf{Y}}, \tilde{\mathbf{Y}}_i \rangle$. Note that $\tilde{\mathbf{Y}}_i \sim \text{Unif}\big(\mathcal{S}^{d-2}\big)$ is equivalent to $\tilde{\mathbf{Y}}_i$ following the vMF distribution $p_{0}(\cdot | \tilde{\mathbf{Y}})$ with $\kappa = 0$, and dimension $d-1$. Therefore, by Lemma \ref{lem:prelim_results}\ref{lem:fidelity_dist}, the probability density function of $X_i$ is given by $p_{0,d-1}(x) = \tilde{C}_{d-1}(0) (1-x^2)^{\frac{d-4}{2}}$.

    \item[(iii)] We first show that the posterior distribution $q_{\kappa}(\mathbf{h} | \mathbf{m})$ matches the vMF distribution $p_\kappa(\mathbf{h} | \mathbf{m})$ (as given in Eq. \eqref{eq:vMF_dist}). Using Bayes' rule:
    \begin{align*}
        q_{\kappa}(\mathbf{h} | \mathbf{m}) \propto p_\kappa(\mathbf{m} | \mathbf{h}) p(\mathbf{h}) \propto e^{\kappa \langle \mathbf{h}, \mathbf{m} \rangle}.
    \end{align*}
    According to Eq. \eqref{eq:vMF_dist}, the normalization constant in the above equation is given by $C_d(\kappa)$ which gives us,
    \begin{align*}
        q_{\kappa}(\mathbf{h} | \mathbf{m}) = C_d(\kappa) e^{\kappa \langle \mathbf{h}, \mathbf{m} \rangle}.
    \end{align*}

    Next, we show the simplification for the KL-divergence of the conditional distribution $q_{\kappa}(\cdot | \mathbf{m})$ from $p(\cdot)$. We have
    \begin{align*}
    D_{\mathrm{KL}}\big(q_{\kappa}(\cdot | \mathbf{m}) \| p(\cdot)\big) &= \EE_{\mathbf{h} \sim q_{\kappa}(\cdot | \mathbf{m})} \left[\log \frac{q_{\kappa}(\mathbf{h} | \mathbf{m})}{p(\mathbf{h})}\right] \\
    &= \EE_{\mathbf{h} \sim q_{\kappa}(\cdot | \mathbf{m})} \left[\log \left(\frac{C_d(\kappa) e^{\kappa \langle \mathbf{h}, \mathbf{m} \rangle}}{C_d(0)}\right)\right]\\
    &= \EE_{\mathbf{h} \sim q_{\kappa}(\cdot | \mathbf{m})} \left[\log C_d(\kappa) - \log C_d(0) + \kappa \langle \mathbf{h}, \mathbf{m} \rangle\right] \\
    &= \log C_d(\kappa) - \log C_d(0) + \kappa \cdot \EE_{\mathbf{h} \sim q_{\kappa}(\cdot | \mathbf{m})} [\langle \mathbf{h}, \mathbf{m} \rangle].
    \end{align*}
    Further, by Law of Iterated Expectation,
    \begin{align*}
        \EE_{\mathbf{m}} \left[ \EE_{\mathbf{h} \sim q_{\kappa}(\cdot | \mathbf{m})} [\langle \mathbf{h}, \mathbf{m} \rangle] \right] = \EE_{\mathbf{h},\mathbf{m}}[\langle \mathbf{h}, \mathbf{m} \rangle] = \EE[W],
    \end{align*}
    where $W$ is a random variable as defined in Lemma \ref{lem:prelim_results}\ref{lem:fidelity_dist}. Substituting this back in the calculation of KL-divergence, we have
    \begin{align*}
        \EE_{\mathbf{m}} \big[D_{\mathrm{KL}}\big(q_{\kappa} (\cdot | \mathbf{m})\| p(\cdot)\big) \big] = \kappa \EE[W] - \log \frac{C_d(0)}{C_d(\kappa)}.
    \end{align*}
\end{itemize}
\end{proof}

\subsection{Concentration of Message Fidelity}

\begin{lemma}[Concentration of $W$]
\label{lem:w_concentration}
Let $W\in[-1,1]$ follows the density function $p_{\kappa,d}(w)$ (see Lemma \ref{lem:prelim_results}). Suppose $\rho\in[0,1)$ is fixed and $\kappa=\frac{\rho}{1-\rho^2}d^{\star}$ with $d^{\star} = d-3$, then, for all $d^{\star} >0$,
\begin{align}
\label{eq:W-subgauss}
\mathbb{E}\Big[\exp\big(t (W - \mathbb{E}[W]) \big)\Big] \le \exp\left( \frac{t^2}{2 d^\star} \right), &&
\mathbb{P}\big(|W - \mathbb{E}[W]| > \varepsilon \big) \le 2 \exp\left( -\frac{d^\star \varepsilon^2}{2} \right).
\end{align}
Further, we have that there exists a universal constant $K$ such that,
\begin{align*}
   \frac{K (1-\rho^2)}{\sqrt{d^\star}} \exp \left( \frac{d^\star\rho^2}{1-\rho^2} + \frac{d^\star}{2} \log(1-\rho^2)\right) \leq \int_{-1}^1 \exp \left( \frac{d^\star\rho w}{1-\rho^2} + \frac{d^\star}{2} \log(1-w^2)\right) dw.
\end{align*}
and
\begin{align*}
    \int_{-1}^1 \exp \left( \frac{d^\star\rho w}{1-\rho^2} + \frac{d^\star}{2} \log(1-w^2)\right) dw \leq \sqrt{\frac{2\pi}{d^\star}} \exp \left( \frac{d^\star\rho^2}{1-\rho^2} + \frac{d^\star}{2} \log(1-\rho^2)\right).
\end{align*}
Finally, we also have that, there exits a universal constant $\bar K$ such that
\begin{align*}
    \mathbb{E}[|W-\rho|] \leq \frac{\bar K}{\sqrt{d^\star (1-\rho^2)}}.
\end{align*}
\end{lemma}

\begin{proof}
We prove the result in three parts.
\begin{enumerate}
    \item \textit{Concentration of $W$ around the mean:} Consider the potential function $I(w)$ such that the density is expressed as $p_{\kappa,d}(w) \propto e^{-I(w)}$, where
\begin{equation*}
I(w) := -\kappa w - \frac{d^\star}{2} \log(1-w^2)
\end{equation*}
We first compute the second derivative of $I(w)$ on the support $(-1, 1)$:
\begin{equation*}
I'(w) = -\kappa + \frac{d^\star w}{1-w^2}, \quad I''(w) = d^\star \left( \frac{1+w^2}{(1-w^2)^2} \right) \geq d^\star.
\end{equation*}
As case be easily observed, we have the uniform lower bound on the second derivative,
\begin{equation*}
I''(w) \ge I''(0) = d^\star.
\end{equation*}
By the Bakry-\'Emery theorem, the distribution of $W$ satisfies a Log-Sobolev Inequality with constant $ 1/d^\star$. It follows from Herbst's argument that $W$ is sub-Gaussian with parameter $ 1/d^\star$, i.e., we have
\begin{align*}
    \mathbb{E}\Big[\exp\big(t (W - \mathbb{E}[W]) \big)\Big] \le \exp\left( \frac{t^2}{2 d^\star} \right).
\end{align*}
We refer the readers to \citep[Theorem 5.2]{ledoux2001concentration} and \citep{bakry2013analysis} for more details on Bakry-\'Emery theorem and Herbst's argument. Next, the tail bound follows from a standard Chernoff's bound.

\item \textit{Bound on the integral:} As already shown, $I''(w) \geq d^\star$, which gives us that
\begin{align*}
    I(w) \geq I(\rho) + I'(\rho)(w-\rho) + \frac{1}{2} d^\star (w-\rho)^2, \ \ \ \forall w\in(-1,1).
\end{align*}
Under the condition that $\kappa=\frac{\rho}{1-\rho^2}d^{\star}$, we get that $I'(\rho) = 0$. As such, for all $w\in(-1,1)$,
\begin{align}
\label{eq:Iw_lower_bound}
    I(w) \geq I(\rho) + \frac{1}{2} d^\star (w-\rho)^2.
\end{align}
This in turn implies that
\begin{align*}
    \int_{-1}^1 e^{-I(w)} dw \leq e^{-I(\rho)} \int_{-\infty}^\infty \exp\left(-\frac{1}{2} d^\star (w-\rho)^2\right) dw = \sqrt{\frac{2\pi}{d^\star}}e^{-I(\rho)}.
\end{align*}
We define
\begin{equation*}
    \hat{\rho} = \max \Big\{ \rho , \frac{1}{\sqrt{d^\star}} \Big\},
\end{equation*}
and we consider the neighborhood $(-\hat{\rho}, \hat{\rho})$. We also define
\begin{align*}
    H_{\max} := \frac{1}{d^\star}\max_{w \in (-\hat{\rho}, \hat{\rho})} I''(w) = \frac{1}{d^\star} I''(\hat \rho) = \frac{1+\hat \rho^2}{(1-\hat \rho^2)^2}.
\end{align*}
Then, by Taylor's theorem, for all $w\in (-\hat{\rho}, \hat{\rho})$, there exists a $\xi \in [w,\rho]$, such that
\begin{align*}
    I(w) &= I(\rho) + I'(\rho)(w-\rho) + \frac{1}{2} I''(\xi) (w-\rho)^2\\
    & \stackrel{(a)}{=}I(\rho) + \frac{1}{2} I''(\xi) (w-\rho)^2\\
    & \leq I(\rho) + \frac{1}{2} \Big( \max_{\xi \in [w,\rho]}|I''(\xi)| \Big) \times (w-\rho)^2\\
    &\leq I(\rho) + \frac{1}{2} \Big( \max_{\xi \in (-\hat{\rho}, \hat{\rho})}|I''(\xi)| \Big) \times (w-\rho)^2\\
    & = I(\rho) + \frac{1}{2} d^\star H_{\max} (w-\rho)^2,
\end{align*}
where (a) uses $I'(\rho) = 0$. Substituting this into the integral yields:
\begin{align*}
\int_{-1}^1 e^{-I(w)} dw &\ge \int_{- \hat \rho }^{\hat \rho} e^{-I(w)}   dw\\
&\geq \int_{- \hat \rho }^{\hat \rho} \exp\left(-I(\rho) -\frac{1}{2} d^\star H_{\max} (w-\rho)^2\right) dw\\
&= e^{-I(\rho)} \times \frac{1}{\sqrt{d^\star H_{\max}}} \int_{- (\hat \rho + \rho)\sqrt{d^\star H_{\max}} }^{(\hat \rho - \rho) \sqrt{d^\star H_{\max}}} \exp\left(-\frac{1}{2} u^2\right) du
\end{align*}
where we used the substitution $u = \sqrt{d^{\star} H_{\max}}(w-\rho)$. Here, we have that
\begin{align*}
    (\hat \rho - \rho) \sqrt{d^\star H_{\max}} = \Big( \max \Big\{ \rho , \frac{1}{\sqrt{d^\star}} \Big\} - \rho\Big) \sqrt{d^\star H_{\max}} \ge 0,
\end{align*}
and
\begin{align*}
    (\hat \rho + \rho) \sqrt{d^\star H_{\max}} &= \Big( \max \Big\{ \rho , \frac{1}{\sqrt{d^\star}} \Big\} + \rho\Big) \sqrt{d^\star H_{\max}} \geq \Big( 1 + \rho\sqrt{d^\star }\Big) \sqrt{\frac{1+\rho^2}{(1-\rho^2)^2}} \geq 1.
\end{align*}
Substituting these in the previous equation gives us that
\begin{align*}
    \int_{-1}^1 e^{-I(w)} dw &\ge e^{-I(\rho)} \times \frac{1}{\sqrt{d^\star H_{\max}}} \int_{-1}^{0} \exp\left(-\frac{1}{2} u^2\right) du\\
    &= e^{-I(\rho)} \frac{1-\rho^2}{\sqrt{d^\star(1+\rho^2)}} \int_{-1}^{0}\exp\left(-\frac{1}{2} u^2\right) du \geq \frac{K(1-\rho^2)}{\sqrt{d^\star}}e^{-I(\rho)},
\end{align*}
where we choose $K = \frac{1}{\sqrt{2}} \int_{-1}^{0}\exp\left(-\frac{1}{2} u^2\right) du$.

\item \textit{Second moment around $\rho$:} We have
\begin{align*}
    \mathbb{E}[(W-\rho)^2] &= \Big[ \int_{-1}^1 \exp \left( -I(w) \right) dw \Big]^{-1} \int_{-1}^1 (w-\rho)^2 \exp \left( -I(w)\right) dw \allowdisplaybreaks\\
    & \stackrel{(a)}{\leq} \frac{\sqrt{d^\star}}{K (1-\rho^2)} e^{I(\rho)} \int_{-1}^1 (w-\rho)^2 \exp \left( -I(w)\right) dw\allowdisplaybreaks\\
    & \stackrel{(b)}{\leq} \frac{\sqrt{d^\star}}{K (1-\rho^2)} e^{I(\rho)} \int_{-1}^1 (w-\rho)^2 \exp \left( -I(\rho) -\frac{1}{2} d^\star (w-\rho)^2\right) dw\allowdisplaybreaks\\
    & \leq \frac{\sqrt{d^\star}}{K (1-\rho^2)} \int_{-\infty}^\infty (w-\rho)^2 \exp \left( -\frac{1}{2} d^\star (w-\rho)^2\right) dw\allowdisplaybreaks\\
    & = \frac{\sqrt{d^\star}}{K (1-\rho^2)} \sqrt{\frac{2\pi}{(d^\star)^3}} = \frac{\sqrt{2\pi}}{K d^\star (1-\rho^2)},
\end{align*}
where (a) follows by the bound on $\int_{-1}^1 \exp \left( -I(w) \right) dw $ provided in second part, and (b) follows by Eq. \eqref{eq:Iw_lower_bound}. Now the result follows by using $\mathbb{E}|W-\rho| \leq \sqrt{\mathbb{E} [(W-\rho)^2]}$ and choosing $\bar K = \big( \sqrt{2\pi}/ K \big)^{1/2}$.
\end{enumerate}
\end{proof}

\section{Proofs of Results in Section \ref{sec:limiting_d_results}}

\subsection{Proof of Proposition \ref{prop: utility_approx}}
\label{sec:proof_prop_util_approx}

\begin{proof}
Recall the representation (Eq.~\eqref{eq:util_representation}), $\langle \mathbf{h},\boldsymbol{\theta}_i\rangle  =  W W_i  +  \sqrt{1-W^2} \sqrt{1-W_i^2} X_i,$ where, conditional on the message $\mathbf{m}$,
\begin{align*}
W=\langle \mathbf{h},\mathbf{m}\rangle,\qquad W_i=\langle \boldsymbol{\theta}_i,\mathbf{m}\rangle,\qquad X_i=\langle \mathbf{Y},\mathbf{Y}_i\rangle,\quad \langle \mathbf{Y},\mathbf{m}\rangle=\langle \mathbf{Y}_i,\mathbf{m}\rangle=0,
\end{align*}

Let $\rho\in[0,1)$ and $\alpha\geq 0$ be fixed and suppose $\kappa = \frac{\rho}{1-\rho^2}d^{\star}$, where $d^{\star}=d-3$. We first bound the difference between the actual utility and the utility evaluated at the concentration point $\rho$. For simplicity of notations, we define $U_i(W) := W W_i + \sqrt{1-W^2} \sqrt{1-W_i^2} X_i$. Using the fact that $W_i , X_i \in [-1,1]$ and the triangle inequality, we have
\begin{align*}
\Big| U_i(W) - U_i(\rho) \Big|
&\leq |W_i| |W-\rho| + \sqrt{1-W_i^2} |X_i| \left| \sqrt{1-W^2} - \sqrt{1-\rho^2} \right| \allowdisplaybreaks\\
&\leq |W-\rho| + \left| \sqrt{1-W^2} - \sqrt{1-\rho^2} \right| \allowdisplaybreaks\\
& \le |W-\rho| \left( 1 + \frac{2}{\sqrt{1-\rho^2}} \right) \le  \frac{3|W-\rho|}{\sqrt{1-\rho^2}},
\end{align*}
where we use the identity for the difference of square roots:
\begin{align}
\label{eq:sqrt_1_minus_w_sq}
\left| \sqrt{1-W^2} - \sqrt{1-\rho^2} \right| = \frac{|W^2 - \rho^2|}{\sqrt{1-W^2} + \sqrt{1-\rho^2}} \leq \frac{|W+\rho||W-\rho|}{\sqrt{1-\rho^2}} \leq \frac{2|W-\rho|}{\sqrt{1-\rho^2}}.
\end{align}
From Lemma \ref{lem:w_concentration}, we have that
\begin{align*}
   \mathbb{E}[|W-\rho|] \leq \frac{\bar K}{\sqrt{d^\star (1-\rho^2)}}.
\end{align*}
By integrating these results, the expected maximum utility satisfies:
\begin{align}
\label{eq:W_to_rho_max_util}
    \mathbb{E} \Big[ \big| \max_i U_i(W) - \max_i U_i(\rho) \big| \Big] \leq \frac{3 \bar K}{\sqrt{d^\star}(1-\rho^2)}.
\end{align}

We define the simplified utility function as,
\begin{align} \label{eq:utility_def}
V(w,x) := \rho w + \sqrt{1-\rho^2} \sqrt{1-w^2} x.
\end{align}
We define the tail probability of $G(t)$ as the probability that a single sample ($W_1,X_1$) yields a utility of at least $t$,
\begin{align}
\label{eq:tail_prob_def}
G(t) := \mathbb{P}\left( V(W, X) \geq t \right) = \iint_{\mathcal{A}_t} p_{\kappa, d}(w) \times p_{0, d-1}(x) dw  dx,
\end{align}
where the set $\mathcal{A}_t$ is the region of the domain satisfying the utility threshold, i.e., $\mathcal{A}_t := \{ (w, x) \in (-1,1)^2 : V(w, x) \geq t \}.$ Finally, we denote the minimum rate required to achieve utility $t$ as $J(t)$, which corresponds to the optimization problem:
\begin{align} \label{eq:min_rate_def}
J(t) := \inf_{w,x \in (-1,1)} \big\{I_\rho(w, x) \ \text{ such that } V(w,x) \geq t \big\}.
\end{align}
Using the closed-form expression for the product of the marginal densities  and , we can express the joint density in terms of the large deviations rate function:
\begin{align}
\label{eq:pdf_G}
p_{\kappa, d}(w) p_{0, d-1}(x) = \frac{A_{\rho, d}}{\sqrt{1-x^2}} \exp\left(-d^{\star} I_{\rho}(w,x)\right),
\end{align}
where the normalization constant is given by
\begin{align*}
A_{\rho, d}^{-1} = \int_{-1}^1 \int_{-1}^1 \frac{1}{\sqrt{1-x^2}} \exp\left(-d^{\star} I_{\rho}(w,x)\right)dw dx.
\end{align*}
We prove the main result in multiple parts. The first part is the following claim that uses the result in Lemma \ref{lem:w_concentration}.

\begin{claim}
\label{claim:A_rho_d_bound}
Suppose $K$ is a universal constant as mentioned in Lemma \ref{lem:w_concentration}. Then, we have
\begin{align*}
   \frac{d - 4}{2\pi} \leq A_{\rho, d} \leq \frac{d -3}{K^2}.
\end{align*}
\end{claim}

\begin{proof}[Proof of Claim \ref{claim:A_rho_d_bound}]
We have that
\begin{align*}
    \frac{1}{A_{\rho, d}} &= \int_{-1}^1 \int_{-1}^1 \frac{1}{\sqrt{1-x^2}} \exp\left(-d^{\star} I_{\rho}(w,x)\right)dw dx\\
    & = \int_{-1}^1 \int_{-1}^1  \exp\left(\frac{d^{\star} \rho(w - \rho)}{1-\rho^2} + \frac{d^{\star}}{2} \log\frac{(1-w^2)}{(1-\rho^2)} + \frac{d^{\star}-1}{2} \log(1-x^2) \right)dw dx\\
    & = \int_{-1}^1  \exp\left(\frac{d^{\star} \rho(w-\rho)}{1-\rho^2} + \frac{d^{\star}}{2} \log\frac{(1-w^2)}{(1-\rho^2)} \right)dw \int_{-1}^1   \exp\left( \frac{d^{\star}-1}{2} \log(1-x^2) \right)dx.
\end{align*}
Then, from Lemma \ref{lem:w_concentration}, we know that
\begin{align*}
   K \frac{1-\rho^2}{\sqrt{d^\star}}\leq \int_{-1}^1  \exp\left(\frac{d^{\star} \rho(w-\rho)}{1-\rho^2} + \frac{d^{\star}}{2} \log\frac{(1-w^2)}{(1-\rho^2)} \right)dw \leq \sqrt{\frac{2\pi}{d^\star}}
\end{align*}
Next, by simply substituting $\rho=0$ and $d^{\star} \rightarrow d^{\star}-1$ in the previous equation,
\begin{align*}
   \frac{K}{\sqrt{d^\star-1}} \leq  \int_{-1}^1   \exp\left( \frac{d^{\star}-1}{2} \log(1-x^2) \right)dx \leq \sqrt{\frac{2\pi}{d^\star -1}}.
\end{align*}
Thus,
\begin{align*}
   \frac{K^2}{d^\star} \leq \frac{1}{A_{\rho, d}} \leq \frac{2\pi}{ d^\star -1}.
\end{align*}
This completes the proof of the claim.
\end{proof}

Next we make the claim presenting an upper bound on the probability $G(t)$.
\begin{claim}
\label{claim:g_upper_bound}
There exists a universal constant $K_1$ such that,
\begin{align*}
    G(t) \leq K_1 d^{\star} \exp\left(-d^{\star} J(t)\right).
\end{align*}
\end{claim}

\begin{proof}[Proof of Claim \ref{claim:g_upper_bound}]
From the definition of $J(t)$, we have that $I_\rho(w,x)\geq J(t)$ for all $(w,x) \in \mathcal{A}_t$. This gives us that
\begin{align*}
G(t) &\leq \exp\left(-d^{\star} J(t)\right) \cdot A_{\rho, d} \iint_{\mathcal{A}_t} \frac{1}{\sqrt{1-x^2}} dw dx \leq 2\pi A_{\rho, d}\exp \left(-d^{\star} J(t)\right),
\end{align*}
where we use that
\begin{align*}
    \iint_{\mathcal{A}_t} \frac{1}{\sqrt{1-x^2}} dw dx \leq \int_{-1}^1 \int_{-1}^1 \frac{1}{\sqrt{1-x^2}} dw dx = 2\pi.
\end{align*}
Afterwards, the result follows by using that $A_{\rho, d} \leq \frac{d-3}{K^2} = \frac{d^{\star}}{K^2}$, and choosing the constant $K_1 = \frac{2\pi}{K^2}$. This completes the proof of the claim.
\end{proof}

Next, we prove the lower bound on the probability $G(t)$.
\begin{claim}
\label{claim:g_lower_bound}
There exists a constant $K_2$ such that,
\begin{align*}
    G(t) \geq \frac{K_2}{\sqrt{d^\star}} \exp\left(-d^{\star} J(t)\right).
\end{align*}
\end{claim}

\begin{proof}[Proof of Claim \ref{claim:g_lower_bound}]
Let $\mathbf{z}_t = (w_t, x_t)$ be the point on the boundary $V(w,x) =t$ that minimizes the rate function $I_\rho(w,x)$. By the first-order optimality conditions, we have that $\nabla I_\rho(\mathbf{z}_t)$ must be parallel to $\nabla V(\mathbf{z}_t)$. We define an orthonormal basis ($\mathbf{e}_{\perp}, \mathbf{e}_{\parallel})$ centered at $\mathbf{z}_t$ as follows,
\begin{align}
\mathbf{e}_{\perp} := \frac{\nabla V(w_t, x_t)}{\|\nabla V(w_t, x_t)\|}, \quad \mathbf{e}_{\parallel} := \begin{pmatrix} -e_{\perp, x} \ e_{\perp, w} \end{pmatrix}.
\end{align}

Here, $\mathbf{e}_{\perp}$ is the unit vector normal to the level set $\{V(w,x) = t\}$ pointing into the set $\mathcal{A}_t$, and $\mathbf{e}_{\parallel}$ is the corresponding unit tangent vector. Any point $\mathbf{z}= (w,x)$ in a neighborhood of $\mathbf{z}_t$ can be represented in these local coordinates ($r,s$) as
\begin{align}
\mathbf{z} = \mathbf{z}_t + r \mathbf{e}_{\perp} + s \mathbf{e}_{\parallel},
\end{align}
Since the transformation is a rotation followed by a translation, the Jacobian determinant is $|\text{det}(\mathbf{e}_{\perp}, \mathbf{e}_{\parallel})| = 1$, ensuring that the area element satisfies $dw dx = dr ds$.

Consider a neighborhood $\mathcal{N}_d$ around $\mathbf{z}_t$ defined in local coordinates by
\begin{align*}
  \mathcal{N}_d(\mathbf{z}_t) = \big\{ \mathbf{z} = \mathbf{z}_t + r \mathbf{e}_{\perp} + s \mathbf{e}_{\parallel} : |r| \leq d^{-1/2}, \ |s| \leq d^{-1/2} \big\}
\end{align*}
Since $\mathbf{z}_t$ lies in the interior of the domain $(-1, 1)^2$, the utility function $V$ is twice-continuously differentiable on $\mathcal{N}_d(\mathbf{z}_t)$ for sufficiently large $d$. Because the Hessian of function $V$, $\mathbf{H}_V$ is continuous, its eigenvalues are bounded on this compact neighborhood. Thus, there exists a constant $K_t > 0$, independent of $d$, such that the minimum eigenvalue of $\mathbf{H}_V$ satisfies
\begin{align*}
    \lambda_{\min}(\mathbf{H}_V(\boldsymbol{\xi})) \geq -K_t \quad \text{for all } \boldsymbol{\xi} \in \mathcal{N}_d.
\end{align*}

Applying Taylor's theorem at $\mathbf{z}_t$, for any point $\mathbf{z} \in \mathcal{N}_d(\mathbf{z}_t)$, there exists some $\boldsymbol{\xi}$ such that
\begin{align*}
    V(r, s) &= V(\mathbf{z}_t) + \nabla V(\mathbf{z}_t)^\top (r \mathbf{e}_{\perp} + s \mathbf{e}_{\parallel}) + \frac{1}{2} (r \mathbf{e}_{\perp} + s \mathbf{e}_{\parallel})^\top \mathbf{H}_V(\boldsymbol{\xi}) (r \mathbf{e}_{\perp} + s \mathbf{e}_{\parallel})\\
    &\geq t + \|\nabla V(\mathbf{z}_t)\| r - \frac{K_t}{2}(r^2 + s^2),
\end{align*}
where we use the fact that $V(\mathbf{z}_t) = t$, $\nabla V(\mathbf{z}_t) = \|\nabla V(\mathbf{z}_t)\| \mathbf{e}_{\perp}$, and using the lower bound on the Hessian eigenvalues. To ensure that a point $\mathbf{z} = \mathbf{z}_t + r \mathbf{e}_{\perp} + s \mathbf{e}_{\parallel}$ is contained in the  set $\mathcal{A}_t = \{ \mathbf{z} : V(\mathbf{z}) \geq t \}$, it is sufficient to satisfy
\begin{align*}
    \|\nabla V(\mathbf{z}_t)\| r \geq \frac{K_t}{2}(r^2 + s^2).
\end{align*}
As such, we have that for all sufficiently large $d$
\begin{align*}
    \mathcal{R}_t := \Big\{ \mathbf{z} = \mathbf{z}_t + r \mathbf{e}_{\perp} + s \mathbf{e}_{\parallel} : (r,s) \in \mathcal{N}_t,  \Big\}  \subseteq \mathcal{A}_t,
\end{align*}
where
\begin{align*}
    \mathcal{N}_t = \Big\{ (r,s) : |r| \leq d^{-1/2}, \ |s| \leq d^{-1/2}, \|\nabla V(\mathbf{z}_t)\| r \geq \frac{K_t}{2}(r^2 + s^2) \Big\}.
\end{align*}

Since $I_{\rho}(w, x)$ is twice-continuously differentiable on the interior of the domain, we expand it in the local $(r, s)$ coordinates around the optimizer $\mathbf{z}_t$. Recall that, by the first-order optimality conditions, $\nabla I_{\rho}(\mathbf{z}_t)$ is parallel to $\nabla V(\mathbf{z}_t)$, and thus $\nabla I_{\rho}(\mathbf{z}_t)^\top \mathbf{e}_{\parallel} = 0$. For any $\mathbf{z} \in \mathcal{R}_t$, we have:
\begin{align*}
I_{\rho}(\mathbf{z}) &\leq I_{\rho}(\mathbf{z}_t) + \big|\nabla I_{\rho}(\mathbf{z}_t)^\top (r \mathbf{e}_{\perp} + s \mathbf{e}_{\parallel})\big| + \frac{1}{2} M_t (r^2 + s^2) \\
&\leq J(t) + \Big(\big|\nabla I_{\rho}(\mathbf{z}_t)^\top \mathbf{e}_{\perp}\big| + \frac{M_t \|\nabla V(\mathbf{z}_t)\|}{K_t}\Big) r,
\end{align*}
where we use that $r^2 + s^2 \leq \frac{ \|\nabla V(\mathbf{z}_t)\|}{K_t}r$ for any $\mathbf{z} \in \mathcal{R}_t$. This gives us that
\begin{align}
\label{eq:I_rho_Z_bound}
\exp\big(-d^\star I_{\rho}(\mathbf{z})\big) \geq \exp\big(-d^\star J(t) - d^\star C_t r\big), \ \ \forall \mathbf{z} \in \mathcal{R}_t,
\end{align}
where $C_t := \Big(\big|\nabla I_{\rho}(\mathbf{z}_t)^\top \mathbf{e}_{\perp}\big| + \frac{M_t \|\nabla V(\mathbf{z}_t)\|}{K_t}\Big)$. Now, using Eq. \eqref{eq:tail_prob_def} and \eqref{eq:pdf_G}, we have that
\begin{align*}
G(t) &=  A_{\rho, d} \iint_{\mathcal{A}_t} \frac{1}{\sqrt{1-x^2}}  \exp\big(-d^\star I_{\rho}(w, x)\big) dw dx \\
& \stackrel{(a)}{\geq} A_{\rho, d} \iint_{\mathcal{A}_t} \exp\big(-d^\star I_{\rho}(w, x)\big) dw dx \\
&\stackrel{(b)}{\geq} A_{\rho, d} \iint_{\mathcal{A}_t} \exp\big(-d^\star I_{\rho}(\mathbf{z})\big) dr ds \\
& \stackrel{(c)}{\geq} A_{\rho, d} \iint_{\mathcal{R}_t} \exp\big(-d^\star I_{\rho}(\mathbf{z})\big) dr ds.
\end{align*}
where (a) follows by using $\frac{1}{\sqrt{1-x^2}} \geq 1$, (b) follows by using the earlier argument $|\text{det}(\mathbf{e}_\perp, \mathbf{e}_\parallel)| =1$ implying that $dwdx = drds$, (c) follows as $\mathcal{R}_t \subseteq \mathcal{A}_t$. Using the results in Eq. \eqref{eq:I_rho_Z_bound}, we have that
\begin{align*}
    \iint_{\mathcal{R}_t} \exp\big(-d^\star I_{\rho}(\mathbf{z})\big) dr ds \geq e^{- d^\star J(t)} \iint_{\mathcal{N}_t} \exp\left( - d^\star C_t r\right) drds.
\end{align*}
Next, we establish a lower bound for the integral over $\mathcal{N}_t$. The neighborhood $\mathcal{N}_t$ is defined by the quadratic constraint $\|\nabla V(\mathbf{z}_t)\| r \geq \frac{K_t}{2}(r^2 + s^2)$. We get,
\begin{align*}
    \Big\{ (r,s) : 0\leq r \leq d^{-1/2}, \ |s| \leq d^{-1/2}, r \leq \frac{\|\nabla V(\mathbf{z}_t)\|}{K_t},  \frac{\|\nabla V(\mathbf{z}_t)\|}{K_t} r \geq s^2 \Big\} \subseteq \mathcal{N}_t.
\end{align*}
Thus, writing $\gamma := \frac{2\|\nabla V(\mathbf{z}_t)\|}{K_t}$, we get
\begin{align*}
    \iint_{\mathcal{N}_t} \exp(-d^\star C_t r) dr ds &\geq \int_{0}^{d^{-\frac12}} \exp(-d^\star C_t r) \left( \int_{-\sqrt{\gamma r}}^{\sqrt{\gamma r}} ds \right) dr \\
    & = 2\sqrt{\gamma} \int_{0}^{d^{-\frac12}} \sqrt{r} \exp(-d^\star C_t r)dr\\
    &= \frac{2\sqrt{\gamma}}{(d^\star C_t)^{3/2}} \int_{0}^{\sqrt{d^\star} C_t} \sqrt{u} e^{-u} du.
\end{align*}
As $d \to \infty$, the upper limit $\sqrt{d^\star} C_t$ tends to infinity. The integral converges to the Gamma function value $\Gamma(3/2) = \frac{\sqrt{\pi}}{2}$. As such, for $d^\star$ large enough, we have that there exists a constant $C_1$ such that $\int_{0}^{\sqrt{d^\star} C_t} \sqrt{u} e^{-u} du \geq C_1$. Thus, we get
\begin{align*}
    \iint_{\mathcal{N}_t} \exp(-d^\star C_t r) dr ds \geq \frac{C_2}{(d^\star)^{3/2}},
\end{align*}
where we choose $C_2 = \frac{2\sqrt{\gamma}}{(C_t)^{3/2}} C_1 $. Further, from Claim \ref{claim:A_rho_d_bound}, we have that $A_{\rho,d} \geq \frac{d^\star - 1}{2\pi} \geq \frac{d^\star}{4\pi}$. Thus,
\begin{align*}
    G(t) \geq \frac{K_2}{\sqrt{d^\star}} e^{- d^\star J(t)},
\end{align*}
by choosing $K_2 = \frac{C_2}{4\pi}$. This completes the proof of Claim \ref{claim:g_lower_bound}.
\end{proof}

Next, we present the final steps of the proof. Recall that $\alpha$ denotes the exponent of the recommendation set size. We use $\tilde O(\cdot)$ to denote $O(\cdot)$ up to polylogarithmic factors in $d$.

{\paragraph{Case $\alpha = 0$ (single recommendation).} When $\alpha = 0$, we have $n = \lfloor e^{0} \rfloor = 1$, so the recommendation set consists of a single item and there is no maximum over multiple samples. The expected utility is simply $\mathbb{E}[V(W_1, X_1)]$. Since $V(w,x) = \rho w + \sqrt{1-\rho^2}\sqrt{1-w^2}\, x$ and $W_1, X_1$ are independent with $\mathbb{E}[X_1] = 0$, we have
\begin{align*}
    \mathbb{E}[V(W_1, X_1)] = \rho\, \mathbb{E}[W_1].
\end{align*}
From Lemma~\ref{lem:w_concentration}, $|\mathbb{E}[W_1] - \rho| \leq \mathbb{E}|W_1 - \rho| \leq \frac{\bar K}{\sqrt{d^\star(1-\rho^2)}}$, so
\begin{align*}
    \mathbb{E}[V(W_1, X_1)] = \rho^2 + O(d^{-1/2}).
\end{align*}
Meanwhile, $f(\rho, 0) = \max\{V(w,x) : I_\rho(w,x) \leq 0\}$. Since $I_\rho(w,x) \geq 0$ with equality if and only if $(w,x) = (\rho, 0)$, we have $f(\rho, 0) = V(\rho, 0) = \rho^2$. Combining with Eq.~\eqref{eq:W_to_rho_max_util}, the total approximation error is $O(d^{-1/2})$, which yields the stated bound when $\alpha =0$.

\paragraph{Case $\alpha > 0$.} Next, we assume that $\alpha >0$. } Suppose $t^*$, $t_{\text{High}}$ and $t_{\text{Low}}$ are such that
\begin{align*}
    J(t^*) = \alpha, \ \ J(t_{\text{High}}) = \alpha + \frac{2\log d^\star}{d^\star}, \ \ J(t_{\text{Low}}) = \alpha - \frac{\log d^\star}{d^\star}.
\end{align*}
Then, as $J(t)$ is a strictly increasing function, from Claims \ref{claim:g_lower_bound} and \ref{claim:g_upper_bound}, we have
\begin{align*}
    G(t) \geq K_2 \sqrt{d^\star} \exp\left( - d^\star \alpha \right), \ \ \forall t \leq t_{\text{Low}}, \ \ \text{ and } \ \ G(t) \leq \frac{K_1}{d^\star} \exp\left( - d^\star \alpha \right), \ \ \forall t \geq t_{\text{High}}.
\end{align*}
This in turn implies that, for all $t \geq t_{\text{High}}$,
\begin{align*}
    \mathbb{P}\Big( \max_{i\leq n} V(W_i,X_i) \geq t\Big) = 1-\big(1-G(t)\big)^n \leq n G(t) \leq e^{(d^\star+3) \alpha} \times \frac{K_1}{d^\star} \exp\left( - d^\star \alpha \right) = \frac{K_1}{d^\star}\exp(3\alpha).
\end{align*}
Similarly, for all $t \leq t_{\text{Low}}$,
\begin{align*}
    \mathbb{P}\Big( \max_{i\leq n} V(W_i,X_i) \geq t\Big) \geq 1- \exp\left(- n G(t)\right) \geq 1- \exp\left(- K_2 \sqrt{d^\star}\exp(3\alpha)\right).
\end{align*}
We now bound $\mathbb{E}[\max_i V(W_i,X_i)]$ using a truncation argument. Since $|V(w,x)| \leq 1$ and denoting $V_i := V(W_i,X_i)$, we have
\begin{align*}
    \mathbb{E}\Big[\max_i V_i\Big] &\leq t_{\text{High}} + \mathbb{P}\Big(\max_i V_i > t_{\text{High}}\Big) \leq t_{\text{High}} + \frac{K_1}{d^\star}\exp(3\alpha),\\
    \mathbb{E}\Big[\max_i V_i\Big] &\geq t_{\text{Low}} \cdot \mathbb{P}\Big(\max_i V_i \geq t_{\text{Low}}\Big) - \mathbb{P}\Big(\max_i V_i < t_{\text{Low}}\Big) \geq t_{\text{Low}} - 2\exp\big(-K_2 \sqrt{d^\star}\, e^{3\alpha}\big).
\end{align*}
Since both correction terms are $O(d^{-1})$, we obtain
\begin{align*}
    t_{\text{Low}} - O(d^{-1}) \leq \mathbb{E} \Big[\max_{i\leq n} V(W_i,X_i)\Big] \leq t_{\text{High}} + O(d^{-1}).
\end{align*}
Now, to complete the argument, we need that $t_{\text{High}} - t_{\text{Low}} = \tilde O(d^{-1})$. This holds because for any given $\rho \in [0,1)$, $J(t)$ is a smooth function, and so for any $\alpha\geq 0$, there exists a constant $K_4$ such that $|t_{\text{High}} - t^*| \leq \frac{1}{K_4} |J(t_{\text{High}}) - J(t^*)| = \tilde O(d^{-1})$. Similarly, $|t_{\text{Low}} - t^*| = \tilde O(d^{-1})$. Thus,
\begin{align*}
    \mathbb E \Big[\max_{i\leq n} V(W_i,X_i)\Big] = t^* + \tilde O(d^{-1}).
\end{align*}
Finally, recall that $f(\rho, \alpha) = \sup_{w,x \in (-1,1)} \{ V(w, x) : I_\rho(w, x) \leq \alpha \}$ and $J(t) = \inf_{w,x \in (-1,1)} \{I_\rho(w, x) : V(w, x) \geq t \}$. Note that $f(\rho, \cdot)$ and $J(\cdot)$ are functional inverses of each other. Since $I_\rho$ is strictly convex and $V$ is monotonic in the region of interest, the set $\{ \mathbf{z} : I_\rho(\mathbf{z}) \leq \alpha \}$ is a convex set, and its boundary corresponds exactly to the level sets of $V$ where the maximum is achieved. Thus, $J(t^*) = \alpha$ implies $t^* = f(\rho, \alpha)$, and so
\begin{align*}
    \mathbb E \Big[\max_{i\leq n} V(W_i,X_i)\Big] = f(\rho, \alpha) + \tilde O(d^{-1}).
\end{align*}
Combining with Eq.~\eqref{eq:W_to_rho_max_util}, which gives
\begin{align*}
    \mathbb E \Big[\max_{i\leq n} WW_i + \sqrt{1-W^2} \sqrt{1-W_i^2} X_i\Big] = \mathbb{E}[\max_i V(W_i,X_i)] + O(d^{-1/2}),
\end{align*}
the total error is $\tilde O(d^{-1}) + O(d^{-1/2}) = O(d^{-1/2})$, yielding the $K\sqrt{\log d / d}$ bound stated in the proposition. This completes the proof.

\end{proof}

\subsection{Proof of Proposition \ref{prop: kl_approx}}

\begin{proof}
From Lemma \ref{lem:prelim_results}, we have that
\begin{align*}
    \mathbb E_{\mathbf h,\mathbf m} \left[D_{\mathrm{KL}}\big(q_\kappa(\cdot| \mathbf m) \Vert p(\cdot)\big)\right]
&= -\log\frac{C_d(0)}{C_d(\kappa)}+\kappa \mathbb E[W],
\end{align*}
where $W \sim p_{\kappa, d}(w)$. Further, from Lemma \ref{lem:w_concentration}, we have that
\begin{align}
\label{eq:kappa_w_bound}
    \kappa\mathbb{E} |W - \rho| \leq \frac{\kappa \bar K}{\sqrt{d^\star (1-\rho^2)}} = \frac{\rho \bar K \sqrt{d^\star }}{\sqrt{(1-\rho^2)^3}}.
\end{align}
Note that, the normalization constant $C_d(\kappa)$ is defined as
\begin{align*}
   \frac{1}{C_d(\kappa)} = \int_{\mathbf{h} \in \mathcal{S}^{d-1}} \exp(\kappa \mathbf{h}^\top \mathbf{m}) d\mathbf{m},
\end{align*}
It can be shown that (see \citep{mardia2009directional}),
\begin{align*}
    \frac{1}{C_d(\kappa)} &= \frac{1}{C_{d-1}(0)} \int_{-1}^1 (1-w^2)^{\frac{d^\star}{2}} \exp \left( \kappa w \right) dw \\
    &= \frac{1}{C_{d-1}(0)} \int_{-1}^1 \exp \left( \frac{d^\star\rho w}{1-\rho^2} + \frac{d^\star}{2} \log(1-w^2) \right) dw,
\end{align*}
where $\frac{1}{C_{d-1}(0)}$ is same as the surface area of $\mathcal{S}^{d-2}$ and the second equality follows as $\kappa = \frac{d^\star\rho}{1-\rho^2}$. As such, by using Lemma \ref{lem:w_concentration}, for $\kappa = \frac{\rho d^\star}{1-\rho^2}$, we have
\begin{align*}
   \frac{1}{C_{d-1}(0)} \frac{K (1-\rho^2)}{\sqrt{d^\star}}  \leq \frac{1}{C_{d}(\kappa)} \exp \left( - \frac{d^\star\rho^2}{1-\rho^2} - \frac{d^\star}{2} \log(1-\rho^2)\right) \leq \frac{1}{C_{d-1}(0)} \sqrt{\frac{2\pi}{d^\star}}.
\end{align*}
By substituting $\rho = 0$ in the above expression,
\begin{align*}
   \frac{1}{C_{d-1}(0)} \frac{K}{\sqrt{d^\star}} \leq \frac{1}{C_{d}(0)} \leq \frac{1}{C_{d-1}(0)} \sqrt{\frac{2\pi}{d^\star}}.
\end{align*}
Thus,
\begin{align*}
  -K_0 \leq  -\log \frac{C_{d}(0)}{C_{d}(\kappa)} +  d^\star \left( \frac{\rho^2}{1-\rho^2} +\frac{1}{2} \log(1-\rho^2)\right) \leq K_0 - \log (1-\rho^2),
\end{align*}
where $K_0 =\max\big\{0, \log \frac{\sqrt{2\pi}}{K} \big\}$. This, in turn, implies that
\begin{align*}
  -K_0 \leq  -\log \frac{C_{d}(0)}{C_{d}(\kappa)} +  d^\star \frac{\rho^2}{1-\rho^2} +\frac{d^\star +1}{2} \log(1-\rho^2) \leq K_0,
\end{align*}
Finally, by using Eq. \eqref{eq:kappa_w_bound} and the fact that $\kappa = \frac{\rho d^\star}{1-\rho^2}$, we have
\begin{align*}
    \Big| -\log \frac{C_{d}(0)}{C_{d}(\kappa)} + \kappa \mathbb{E} W -\frac{d^\star +1}{2} \log(1-\rho^2) \Big| \leq K_0 + \frac{\rho \bar K \sqrt{d^\star }}{\sqrt{(1-\rho^2)^3}}.
\end{align*}
Now, the proof is complete by choosing $K_{\ref{prop: kl_approx}} = K_0 + \bar K$ and using $d^\star \leq d$.
\end{proof}

\subsection{Proof of Theorem \ref{thm:asym_opt}}
\label{app:proof_thm_asym_opt}

\begin{proof}
We prove each part of the theorem sequentially.

\paragraph{Part 1: Convergence}
We analyze the asymptotic behavior of each term as $d \to \infty$ under the scaling assumptions $\lambda_s = c_s/d + o(d^{-1})$ and $\lambda_c = c_c/d + o(d^{-1})$. Let $n = \lfloor e^{\alpha d} \rfloor$ and $\kappa = \frac{\rho}{1-\rho^2}(d-3)$ for fixed $\alpha, \rho$. From Proposition \ref{prop: utility_approx}, the expected utility of the best recommendation converges as
\begin{equation*}
   \lim_{d\rightarrow \infty} \mathbb{E}\left[\max_{i \in [n]} \langle \mathbf h, \boldsymbol\theta_i \rangle\right] = f(\rho, \alpha),
\end{equation*}
where $f(\rho, \alpha)$ is the value of the deterministic optimization problem defined in Proposition \ref{prop: utility_approx}. Next, using the scaling of $\lambda_s$, the search cost becomes,
\begin{equation*}
    \lambda_s \log n = \left(\frac{c_s}{d} + o(d^{-1})\right) (\alpha d) = c_s \alpha + o(1).
\end{equation*}
Finally, from Proposition \ref{prop: kl_approx},
\begin{align*}
    D_{\text{KL}} &=  \frac{d-2}{2} \log \left(\frac{1}{1-\rho^2}\right) + \mathcal{E}_{\ref{prop: kl_approx}}(\rho),
\end{align*}
where we use $D_{\text{KL}}$ as a shorthand notation for the KL-divergence. Multiplying by the communication cost parameter $\lambda_c$ and taking the limit,
\begin{align*}
   \lim_{d\rightarrow \infty} \lambda_c D_{\text{KL}} &= \lim_{d\rightarrow \infty} \left(\frac{c_c}{d} + o(d^{-1})\right) \left(\frac{d-2}{2} \log \left(\frac{1}{1-\rho^2}\right) + \mathcal{E}_{\ref{prop: kl_approx}}(\rho)\right) = -\frac{c_c}{2} \log (1-\rho^2),
\end{align*}
where we use the fact that $|\mathcal{E}_{\ref{prop: kl_approx}}(\rho)| \leq \frac{K_{\ref{prop: kl_approx}} \sqrt{d}}{\sqrt{(1-\rho^2)^3}}$, and $K_{\ref{prop: kl_approx}}$ is a universal constant. Combining these terms, the objective function converges to
\begin{equation*}
    \lim_{d \to \infty} \mathcal{P}_d(\kappa, n) = f(\rho, \alpha) - c_s \alpha + \frac{c_c}{2} \log(1-\rho^2).
\end{equation*}
This matches the definition of the objective function in $\mathrm{OPT}_{\mathrm{Joint}}$ (Eq. \eqref{eq:joint_opt}). This shows that the finite-dimensional objective converges uniformly to the asymptotic objective. {Now, under the assumption that the optimizer ($\rho^* , \alpha^*$) lies on a compact set, the maximum also converges, proving the first statement.

As such, to complete the argument, we show that for $c_s > 0$ and $c_c > 0$, any optimizer of $\mathcal{P}_\infty$ must lie in a compact subset of $[0,1) \times [0,\infty)$. Since $f(\rho, \alpha) \leq 1$, the objective satisfies
\begin{align*}
    \mathcal{P}_\infty(\rho, \alpha) \leq 1 - c_s \alpha + \frac{c_c}{2} \log(1-\rho^2).
\end{align*}
Note that $\mathcal{P}_\infty(0, 0) = f(0,0) = 0$ provides a baseline value. For any $(\rho, \alpha)$ to be optimal, we need $\mathcal{P}_\infty(\rho, \alpha) \geq 0$, which requires both $c_s \alpha \leq 1$ and $\frac{c_c}{2} \log(1-\rho^2) > -1$. The first condition gives $\alpha \leq 1/c_s$. The second gives $1-\rho^2 > e^{-2/c_c}$, i.e., $\rho < \sqrt{1 - e^{-2/c_c}} < 1$. Since the optimizer is confined to the compact set $[0, \sqrt{1 - e^{-2/c_c}}] \times [0, 1/c_s]$, the uniform convergence of $\mathcal{P}_d$ to $\mathcal{P}_\infty$ on this set implies convergence of the optima.}

\paragraph{Part 2: Monotonicity}
We denote the asymptotic objective function by $\mathcal{P}_\infty(\rho, \alpha)$ which is given by
\begin{equation*}
    \mathcal{P}_\infty(\rho, \alpha) = f(\rho, \alpha) - c_s \alpha + \frac{c_c}{2} \log(1-\rho^2).
\end{equation*}
For simplicity of the argument, we assume that the optimal solution $(\rho^*, \alpha^*)$ lies in the interior of the feasible region $(0, 1) \times (0, \infty)$. From the first-order necessary conditions for a local maximum, we have
\begin{align}
\label{eq:foc}
    \frac{\partial}{\partial \rho} \mathcal{P}_\infty = \frac{\partial f}{\partial \rho} - \frac{c_c \rho}{1-\rho^2} = 0, && \frac{\partial}{\partial \alpha} \mathcal{P}_\infty = \frac{\partial f}{\partial \alpha} - c_s = 0
\end{align}
For the solution to be a local maximum, the second-order sufficient conditions require that
\begin{align*}
    \frac{\partial^2}{\partial \rho^2} \mathcal{P}_\infty <0, && \frac{\partial^2}{\partial \alpha^2} \mathcal{P}_\infty <0,
\end{align*}
and the Hessian matrix $\mathbf{H}_\infty$ of the objective function $\mathcal{P}_\infty$ with respect to $\rho$ and $\alpha$ be negative definite, where the Hessian is given by
\begin{equation*}
    \mathbf{H}_\infty
    = \begin{pmatrix}
        \frac{\partial^2 f}{\partial \rho^2} - \frac{c_c(1+\rho^2)}{(1-\rho^2)^2} & \frac{\partial^2 f}{\partial \rho \partial \alpha} \\
        \frac{\partial^2 f}{\partial \alpha \partial \rho} & \frac{\partial^2 f}{\partial \alpha^2},
    \end{pmatrix}
\end{equation*}
where second order conditions imply that $|\mathbf{H}_\infty | > 0$. Furthermore, we invoke the property that communication precision and search set size act as substitutes in the production of utility, simply because as $\rho$ increases, the alignment between the agent's recommendation and the user's preference increases. This reduces the necessity of a large search set exponent $\alpha$ to find a high-utility outcome, as such, we have $\frac{\partial^2}{\partial \alpha \partial \rho } \mathcal{P}_\infty = \frac{\partial^2 f}{\partial \rho \partial \alpha} < 0$.

Since the first-order conditions must hold even as $c_c$ changes, we can take the total derivative of both equations in Eq. \eqref{eq:foc} with respect to $c_c$ to get that

\begin{equation*}
    \mathbf{H}_\infty
    \begin{pmatrix}
        \frac{\partial \rho^*}{\partial c_c} \\
        \frac{\partial \alpha^*}{\partial c_c}
    \end{pmatrix}
    +
    \begin{pmatrix}
        \frac{\partial^2}{\partial \rho \partial c_c} \mathcal{P}_\infty \\
        \frac{\partial^2}{\partial \alpha \partial c_c} \mathcal{P}_\infty
    \end{pmatrix}
    =
    \begin{pmatrix}
        0 \\ 0
    \end{pmatrix}.
\end{equation*}
Further, we have that
\begin{align*}
    \frac{\partial^2}{\partial \rho \partial c_c} \mathcal{P}_\infty = \frac{\partial}{\partial c_c} \left( \frac{\partial f}{\partial \rho} - \frac{c_c \rho}{1-\rho^2} \right) = - \frac{\rho}{1-\rho^2}, &&
    \frac{\partial^2}{\partial \alpha \partial c_c} \mathcal{P}_\infty = \frac{\partial}{\partial c_c} \left( \frac{\partial f}{\partial \alpha} - c_s \right) = 0
\end{align*}
Next, we apply Cramer's rule to solve for the linear system, giving us
\begin{equation*}
    \frac{\partial \rho^*}{\partial c_c} = - \frac{1}{|\mathbf{H}_\infty|} \det \begin{pmatrix}
        \frac{-\rho}{1-\rho^2} & \frac{\partial^2}{\partial \alpha \partial \rho } \mathcal{P}_\infty  \\
        0 &
        \frac{\partial^2}{\partial \alpha^2 } \mathcal{P}_\infty
    \end{pmatrix} = \frac{1}{|\mathbf{H}_\infty|} \left( \frac{\rho}{1-\rho^2} \frac{\partial^2}{\partial \alpha^2} \mathcal{P}_\infty \right)< 0,
\end{equation*}
where we use the fact that $|\mathbf{H}_\infty|>0$ and $\frac{\partial^2}{\partial \alpha^2} \mathcal{P}_\infty<0$. Similarly, we get that
\begin{align*}
    \frac{\partial \alpha^*}{\partial c_c} = - \frac{1}{|\mathbf{H}_\infty|} \det \begin{pmatrix}
        \frac{\partial^2}{\partial \rho^2 } \mathcal{P}_\infty & \frac{-\rho}{1-\rho^2} \\
        \frac{\partial^2}{\partial \alpha \partial \rho } \mathcal{P}_\infty & 0
    \end{pmatrix} = \frac{1}{|\mathbf{H}_\infty|} \left( -\frac{\rho}{1-\rho^2} \frac{\partial^2}{\partial \alpha \partial \rho } \mathcal{P}_\infty \right) > 0.
\end{align*}
Next, we repeat the procedure in terms of the search cost parameter $c_s$. We have that
\begin{align*}
    \frac{\partial^2}{\partial \rho \partial c_s} \mathcal{P}_\infty = \frac{\partial}{\partial c_s} \left( \frac{\partial f}{\partial \rho} - \frac{c_c \rho}{1-\rho^2} \right) = 0, &&
    \frac{\partial^2}{\partial \alpha \partial c_s} \mathcal{P}_\infty = \frac{\partial}{\partial c_s} \left( \frac{\partial f}{\partial \alpha} - c_s \right) = -1
\end{align*}
This gives us that
\begin{equation*}
    \frac{\partial \rho^*}{\partial c_s} = \frac{1}{|\mathbf{H}_\infty|} \left( -1 \times \frac{\partial^2}{\partial \alpha \partial \rho } \mathcal{P}_\infty \right)> 0,
\end{equation*}
and also,
\begin{equation*}
    \frac{\partial \alpha^*}{\partial c_s} = \frac{1}{|\mathbf{H}_\infty|} \left( 1 \times \frac{\partial^2}{\partial \rho^2 } \mathcal{P}_\infty \right) < 0.
\end{equation*}
This completes the proof of the monotonicity properties.

\paragraph{Part 3: Search-Only Regime}
To prove this part, we first show that for $c_c > c_s$, the optimal policy for $\mathcal{P}_\infty(\rho,\alpha)$ satisfy that optimal message precision $\rho^* = 0$. From the expression for $f(\rho, \alpha)$ in Proposition \ref{prop: utility_approx}, we know that
\begin{align*}
    f(\rho,\alpha) := \max_{w, x \in (-1,1)} & \ \ \ \rho w + \sqrt{1-\rho^2} \sqrt{1-w^2} x \ \text{ such that } I_{\rho}(w,x) \leq \alpha,
\end{align*}
where $I_{\rho}(w,x)$ is the large deviations rate function, given by
\begin{align*}
    I_{\rho}(w,x) = -\frac{\rho(w-\rho)}{1-\rho^2} - \frac{1}{2} \log(1-w^2)+ \frac{1}{2} \log(1-\rho^2) -\frac{1}{2} \log(1-x^2).
\end{align*}
Using the fact that $-\frac{\rho(w-\rho)}{1-\rho^2} - \frac{1}{2} \log(1-w^2)+ \frac{1}{2} \log(1-\rho^2) \geq 0$ for any $w\in (-1,1)$, we get that
\begin{align*}
    \big\{ (w,x) : |w| < 1, |x|<1, I_{\rho}(w,x) \leq \alpha \big\} \subseteq \Big\{ (w,x) : |w| < 1, |x|<1, -\frac{1}{2} \log(1-x^2)  \leq \alpha \Big\}.
\end{align*}
This implies that
\begin{align*}
    f(\rho,\alpha) \leq \max_{w, x \in (-1,1)} & \ \ \ \rho w + \sqrt{1-\rho^2} \sqrt{1-w^2} x \ \text{ such that } -\frac{1}{2} \log(1-x^2)  \leq \alpha.
\end{align*}
This in turn gives us that,
\begin{align*}
    \mathcal{P}_\infty (\rho, \alpha) &\leq \max_{w, x \in (-1,1)} \ \ \ \rho w + \sqrt{1-\rho^2} \sqrt{1-w^2} x - c_s \alpha - \frac{c_c}{2}\log \frac{1}{1-\rho^2} \ \text{ s.t. } -\frac{1}{2} \log(1-x^2)  \leq \alpha\\
    &= \sqrt{\rho^2 + (1-\rho^2) (1-e^{-2\alpha})} - c_s \alpha - \frac{c_c}{2}\log \frac{1}{1-\rho^2},
\end{align*}
where the second equality is achieved by simply maximizing in terms of $w$ and $x$. It can be shown that when $c_c>c_s$, the $(\rho,\alpha)$ that maximizes the RHS in the above expression satisfy $\rho = 0$. A more similar argument (and in more detail) for the same is presented in Theorem \ref{thm:tilt_opt} and its proof. As such, for $c_c>c_s$, we have
\begin{align*}
    \max_{\rho\in[0,1) , \alpha\geq 0} \mathcal{P}_\infty (\rho, \alpha) \leq \max_{\alpha\geq 0} \sqrt{1-e^{-2\alpha}} - c_s \alpha.
\end{align*}
Further, we also have that,
\begin{align*}
    \max_{\rho\in[0,1) , \alpha\geq 0} \mathcal{P}_\infty (\rho, \alpha) & \geq \max_{ \alpha\geq 0} \mathcal{P}_\infty (0, \alpha) \\
    & = \max_{ \alpha\geq 0} \max_{w, x \in (-1,1)} \Big[\sqrt{1-w^2} x - c_s \alpha \ \text{ s.t. } -\frac{1}{2} \log(1-w^2) -\frac{1}{2} \log(1-x^2)  \leq \alpha \Big]\\
    & = \max_{ \alpha\geq 0}  \max_{x \in (-1,1)} \Big[ x - c_s \alpha \ \text{ s.t. } -\frac{1}{2} \log(1-x^2)  \leq \alpha\Big]\\
    & = \max_{\alpha\geq 0} \sqrt{1-e^{-2\alpha}} - c_s \alpha.
\end{align*}
Combining the above two inequalities, it implies that for $c_c>c_s$,
\begin{align*}
    \max_{\rho\in[0,1) , \alpha\geq 0} \mathcal{P}_\infty (\rho, \alpha) = \max_{\alpha\geq 0} \sqrt{1-e^{-2\alpha}} - c_s \alpha,
\end{align*}
and so $\rho^* = 0$ whenever $c_c>c_s$. Thus, we get that there exists $\bar c_c(c_s)$ such that for $c_c>\bar c_c(c_s)$, we have $\rho^*(c_c,c_s) =0$.
\end{proof}

\section{Proofs of Results in Section \ref{sec:results_finite_d}}

\subsection{Proof of Theorem \ref{thm: perf_gap}}
\label{app:proof_thm_perf_gap}

\begin{proof} We begin by defining the asymptotic objective function derived from the joint optimization problem $\mathrm{OPT}_{\mathrm{Joint}}$. For any $\rho \in [0, 1)$ and $\alpha \ge 0$, with slight abuse of notation, we define
\begin{equation*}
\mathcal{P}_\infty(\rho, \alpha) = f(\rho, \alpha) - c_s \alpha + \frac{c_c}{2} \log(1-\rho^2)
\end{equation*}
where $f(\rho, \alpha)$ is the limit of the expected maximum utility defined in Proposition \ref{prop: utility_approx}. For any fixed $(\rho, \alpha)$ we construct a policy using the mapping $\kappa(\rho) = \frac{\rho}{1-\rho^2}(d-3)$ and $n(\alpha) = e^{d\alpha}$, where we ignore the integrality of the recommendation set size $n$ for simplicity. The finite-dimensional objective $\mathcal{P}_d\big(\kappa(\rho), n(\alpha)\big)$ can be decomposed into three terms: product utility, search cost, and communication cost.

First, invoking Proposition \ref{prop: utility_approx}, the expected product utility satisfies
\begin{equation*}
    \mathbb{E}\left[\max_{i \in [n]} \langle \mathbf h, \boldsymbol \theta_i \rangle\right] = f(\rho, \alpha) + \mathcal{E}_{\ref{prop: utility_approx}}(\rho,\alpha),
\end{equation*}
where $|\mathcal{E}_{\ref{prop: utility_approx}}(\rho,\alpha)| \leq K_{\ref{prop: utility_approx}}(\rho,\alpha) \sqrt{\frac{\log d}{d}}$, Here, we write $\mathcal{E}_{\ref{prop: utility_approx}}(\rho,\alpha)$ instead of simply $\mathcal{E}_{\ref{prop: utility_approx}}$ to show the dependency on $\rho$ and $\alpha$.

Second, utilizing the definition of scaled costs $c_s = d\lambda_s$, the search cost is exactly
\begin{equation*}
    \lambda_s \log n = \frac{c_s}{d} (d \alpha) = c_s \alpha
\end{equation*}
Third, using Proposition \ref{prop: kl_approx}  and the scaled cost $c_c = d\lambda_c$, the communication cost is
\begin{align*}
    \lambda_c D_{\text{KL}} &= \frac{c_c}{d} \left[ \frac{d-2}{2} \log \left(\frac{1}{1-\rho^2}\right) + \mathcal{E}_{\ref{prop: kl_approx}}(\rho) \right] = -\frac{c_c}{2} \log(1-\rho^2) + \frac{c_c}{d}\mathcal{E}_{\ref{prop: kl_approx}}(\rho) - \frac{c_c}{d} \log \left(\frac{1}{1-\rho^2}\right),
\end{align*}
where $|\mathcal{E}_{\ref{prop: kl_approx}}(\rho)| \leq K_{\ref{prop: kl_approx}} \frac{\sqrt{d}}{(1-\rho^2)^3}$, and $K_{\ref{prop: kl_approx}}$ is a universal constant.

Combining these terms, we can relate the finite objective $\mathcal{P}_d$ to the asymptotic objective $\mathcal{P}_\infty$ for any valid $(\rho, \alpha)$ pair
\begin{equation*}
    \mathcal{P}_d(\kappa(\rho), n(\alpha)) = \mathcal{P}_\infty(\rho, \alpha) + \mathcal{E}_{\text{total}}(\rho, \alpha),
\end{equation*}
where the total error is bounded by
\begin{align*}
\mathcal{E}_{\text{total}}(\rho, \alpha) \leq K_{\ref{prop: utility_approx}}(\rho,\alpha) \sqrt{\frac{\log d}{d}} + \frac{c_c}{d} \left(K_{\ref{prop: kl_approx}} \frac{\sqrt{d}}{(1-\rho^2)^3}\right) + \frac{c_c}{d} \log \left(\frac{1}{1-\rho^2}\right),
\end{align*}
giving us that $|\mathcal{E}_{\text{total}}(\rho, \alpha)| $ is dominated by the utility approximation error. And so, there exists a constant $K_{\text{total}}(\rho,\alpha)$ such that for
\begin{align*}
    |\mathcal{E}_{\text{total}}(\rho,\alpha)| \leq K_{\text{total}}(\rho,\alpha)\sqrt{\frac{\log d}{d}}.
\end{align*}

Now, let $(\kappa^*_d, n^*_d)$ be the true optimal solution to $\mathcal{P}_d$. We define the implicit asymptotic parameters $(\tilde{\rho}, \tilde{\alpha})$ by inverting the mapping relations: $\tilde{\alpha} = \frac{1}{d} \log n^*_d$ and $\tilde{\rho}$ such that $\kappa^*_d = \frac{\tilde{\rho}}{1-\tilde{\rho}^2}(d-3)$. The performance gap is given by $\Delta_d := \mathcal{P}_d(\kappa^*_d, n^*_d) - \mathcal{P}_d(\kappa^*_\infty, n^*_\infty)$, where $(\rho^*,\alpha^*)$ is the optimizer of $\mathcal{P}_{\infty}$ and $(k^*_\infty,n^*_\infty)$ is constructed by mapping $(\rho^*,\alpha^*)$ using the mapping $\kappa(\rho)$ and $n(\alpha)$.

Since $(\kappa^*_d, n^*_d)$ is the maximizer of $\mathcal{P}_d$, we immediately have the lower bound $\Delta_d \ge 0$. Next, for the upper bound, we expand both terms using the approximation relation derived above
\begin{align*}
    \big|\mathcal{P}_d(\kappa^*_d, n^*_d) - \mathcal{P}_\infty(\tilde{\rho}, \tilde{\alpha})\big| & \leq K_{\text{total}}(\tilde \rho,\tilde \alpha)\sqrt{\frac{\log d}{d}},\\
    \big|\mathcal{P}_d(\kappa^*_\infty, n^*_\infty) - \mathcal{P}_\infty(\rho^*, \alpha^*)\big|& \leq  K_{\text{total}}(\rho^*,\alpha^*)\sqrt{\frac{\log d}{d}},
\end{align*}
Substituting these into the gap expression yields
\begin{align*}
    \big|\Delta_d \big| & \leq  \big|\mathcal{P}_d(\kappa^*_d, n^*_d) - \mathcal{P}_\infty(\tilde{\rho}, \tilde{\alpha})\big|+ \big|\mathcal{P}_d(\kappa^*_\infty, n^*_\infty) - \mathcal{P}_\infty(\rho^*, \alpha^*)\big|\\
    & \leq \big( K_{\text{total}}(\tilde \rho,\tilde \alpha) + K_{\text{total}}(\rho^*,\alpha^*)\big) \sqrt{\frac{\log d}{d}},
\end{align*}
where we use the fact that $\mathcal{P}_\infty(\tilde{\rho}, \tilde{\alpha}) \le \mathcal{P}_\infty(\rho^*, \alpha^*)$.

It remains to show that $K_{\text{total}}(\tilde{\rho}, \tilde{\alpha})$ is uniformly bounded in $d$. We use a similar argument as in the proof of Theorem \ref{thm:asym_opt}. Since $\mathcal{P}_d(0, 1) \geq 0$ (zero communication, single recommendation), any optimizer must satisfy $\mathcal{P}_d(\kappa^*_d, n^*_d) \geq 0$. Using the bound $\mathbb{E}[\max_i \langle \mathbf{h}, \boldsymbol{\theta}_i \rangle] \leq 1$, we have
\begin{align*}
    0 \leq \mathcal{P}_d(\kappa^*_d, n^*_d) \leq 1 - \lambda_s \log n^*_d - \lambda_c D_{\mathrm{KL}}.
\end{align*}
This gives $\lambda_s \log n^*_d \leq 1$, i.e., $\tilde{\alpha} = \frac{\log n^*_d}{d} \leq \frac{1}{c_s}$. Similarly, $\lambda_c D_{\mathrm{KL}} \leq 1$. From Proposition~\ref{prop: kl_approx}, $D_{\mathrm{KL}} \geq \frac{d-2}{2}\log\frac{1}{1-\tilde{\rho}^2} - |\mathcal{E}_{\ref{prop: kl_approx}}|$. For $d$ sufficiently large, the error term is lower order (in terms of $d$), giving $\frac{c_c}{2}\log\frac{1}{1-\tilde{\rho}^2} \lesssim 1$, which yields $\tilde{\rho} \leq \sqrt{1 - e^{-2/c_c}}$. Thus $(\tilde{\rho}, \tilde{\alpha})$ lies in the compact set $[0, \sqrt{1 - e^{-2/c_c}}] \times [0, 1/c_s]$ for all large $d$, and $K_{\text{total}}(\tilde{\rho}, \tilde{\alpha})$ is uniformly bounded.

This establishes the upper bound and completes the proof.
\end{proof}

\subsection{Proof of Proposition \ref{prop:finite_phase_transition}}

\begin{proof}
From the expression of the objective $\mathcal{P}_d(\kappa,n)$, we have that for any fixed interaction policy $(\kappa, n)$,
\begin{align*}
    \frac{\partial}{\partial \lambda_s} \mathcal{P}_d(\kappa,n) = - \log n, && \frac{\partial}{\partial \lambda_c} \mathcal{P}_d(\kappa,n) = - D_{\text{KL}}\big(q_{\kappa}(\cdot|\cdot) \| p(\cdot)\big).
\end{align*}
By the Envelope Theorem, the derivative of the maximized payoff $ \max_{\kappa, n} \mathcal{P}_d(\kappa, n)$ with respect to a cost parameter is the derivative of the objective evaluated at the optimal policy $(\kappa^*, n^*)$. Since $\log n \geq 0$ for all $n \geq 1$ and the KL divergence is non-negative, it follows that:
\begin{align*}
\frac{\partial}{\partial \lambda_s} \Big[\max_{\kappa, n} \mathcal{P}_d(\kappa, n)\Big] = -\log n^* \leq 0, \quad \text{and} \quad \frac{\partial }{\partial \lambda_c} \Big[\max_{\kappa, n} \mathcal{P}_d(\kappa, n)\Big] = -D_{\text{KL}}(q_{\kappa^*} \| p) \leq 0.
\end{align*}

Next, we use the fact that the expected utility is bounded above by $1$. For any policy with $n \geq 2$, the payoff is constrained by $\mathcal{P}_d(\kappa, n) \leq 1 - \lambda_s \log 2$. We can define a threshold $\bar{\lambda}_s = 1/\log 2$ such that for any $\lambda_s > \bar{\lambda}_s$, the payoff for any $n \geq 2$ is strictly less than zero. Since the system can always achieve a payoff of at least $0$ by selecting $n=1$ and $\kappa=0$, any policy with $n \geq 2$ becomes strictly suboptimal. Combined with the fact that the optimal payoff is a decreasing function in terms of $\lambda_s$, there exists $\bar\lambda_s(\lambda_c)$ such that for all $\lambda_s \geq \bar\lambda_s(\lambda_c)$, we have $n^*_d = 1$. This implies that for $\lambda_c < \underline \lambda_c(\lambda_s) = \bar \lambda_s^{-1}(\lambda_s)$, we have that $n^*_d = 1$.

A similar argument applies to the Pure Search regime. Note that the KL divergence $D_{\text{KL}}(q_{\kappa} \| p)$ is a strictly increasing and unbounded function of $\kappa$ for the vMF distribution. As such, there exists a threshold $\bar{\lambda}_c$ such that for all $\lambda_c > \bar{\lambda}_c$ and for any $n$,
\begin{align*}
    \frac{\partial}{\partial \kappa} \Big[\mathcal{P}_d(\kappa, n)\Big] \Bigg|_{\kappa = 0} < 0.
\end{align*}
Combined with the fact that the optimal payoff is a decreasing function in terms of $\lambda_c$, there exists $\bar\lambda_c(\lambda_s)$ such that for all $\lambda_c \geq \bar\lambda_c(\lambda_s)$, we have $\kappa^*_d = 0$.
\end{proof}

\section{Proofs of Results in Section \ref{sec:tilted_dist}}
\label{app:proofs_tilt}

\subsection{Proof of Theorem \ref{thm:tilt_opt}}

\begin{proof}
Using the decomposition presented in Section \ref{sec:decomposition}, the decomposition of $\boldsymbol{\theta}_i$, as $\boldsymbol{\theta}_i\sim q(\cdot|\mathbf{m}) \in \mathcal{Q}$, we have
\begin{align*}
\mathbf{h} = W \mathbf{m} + \sqrt{1-W^2} \mathbf{Y}, \quad \boldsymbol{\theta}_i = V_i \mathbf{m} + \sqrt{1-V_i^2} \mathbf{Y}_i,
\end{align*}
where $\mathbf{Y}$ and $\mathbf{Y}_i$'s represent the uncommunicated component (see Section \ref{sec:decomposition}). Using the above decomposition, we can write the utility
\begin{align*}
\langle \mathbf{h}, \boldsymbol{\theta}_i \rangle = W V_i + \sqrt{1-W^2}\sqrt{1-V_i^2} X_i,
\end{align*}
where $X_i = \langle \mathbf{Y}, \mathbf{Y}_i \rangle$. From Lemma \ref{lem:prelim_results}, we know that $X_i$'s are independent and identically distribution. The expected utility of the recommendation set is then,
\begin{align*}
U_n(q) = \mathbb{E} \Big[ \max_{i \in [n]} \big\{ W V_i + \sqrt{1-W^2}\sqrt{1-V_i^2} X_i \big\} \Big].
\end{align*}
Using the bound on the MGF (moment generating function) of $W$ from Lemma \ref{lem:w_concentration}, we  have that
\begin{align*}
\mathbb{E}\big[|W-\mathbb E[W]|\big] \leq \frac{1}{\sqrt{d^\star}}.
\end{align*}
Further, as $|W|\leq 1$, we have
\begin{align*}
    \mathbb E \left| \sqrt{1-W^2} - \mathbb E [\sqrt{1-W^2}]  \right| \leq 2 \mathbb E \left| \sqrt{1-W^2} - \sqrt{1-\rho^2}  \right| \leq \frac{2}{\sqrt{1-\rho^2}}\mathbb{E}|W-\rho| \leq  \frac{2\rho \bar K}{\sqrt{d^\star}(1-\rho^2)}.
\end{align*}
where the second inequality follows by using Eq. \eqref{eq:sqrt_1_minus_w_sq}, and the third inequality follows by Lemma \ref{lem:w_concentration}. Further, as $|V_i|\leq 1$ and $|X_i|\leq 1$, we get
\begin{align}
\label{eq:replace_w_expec}
\Big| U_n(q) - \mathbb{E} \Big[ \max_{i \in [n]} \big\{ \mathbb E[W] V_i + \mathbb E \big[\sqrt{1-W^2}\big] \sqrt{1-V_i^2} X_i \big\} \Big] \Big| \leq  \frac{1+2\rho \bar K}{\sqrt{d^\star}(1-\rho^2)}.
\end{align}

To simplify this expectation, we define $M_n = \max_{i \in [n]} X_i$. Then, we can upper bound each individual item's utility using,
\begin{align}
\label{eq:replace_Mn_expec}
\max_{i \in [n]} & \big\{ \mathbb{E}[W] V_i + \mathbb{E}[\sqrt{1-W^2}] \sqrt{1-V_i^2} X_i \big\} \nonumber\\
&\leq \max_{i \in [n]} \big\{ \mathbb{E}[W] V_i + \mathbb{E}[\sqrt{1-W^2}] \sqrt{1-V_i^2} M_n \big\} \nonumber\\
& \leq \max_{i \in [n]} \big\{ \mathbb{E}[W] V_i + \mathbb{E}[\sqrt{1-W^2}] \sqrt{1-V_i^2} \mathbb E [M_n] \big\} + |M_n -\mathbb E [M_n]| \nonumber\\
&\leq \max_v \Big\{\mathbb{E}[W] v + \mathbb{E}[\sqrt{1-W^2}] \sqrt{1-v^2} \mathbb E [M_n] \Big\} + |M_n -\mathbb E [M_n]|,
\end{align}
where we use that $\mathbb{E}[\sqrt{1-W^2}] \sqrt{1-V_i^2}\leq 1$, and the last inequality holds simply because
\begin{equation*}
    \mathbb{E}[W] V_i + \mathbb{E}[\sqrt{1-W^2}] \sqrt{1-V_i^2} \mathbb E [M_n] \leq \max_v \Big\{\mathbb{E}[W] v + \mathbb{E}[\sqrt{1-W^2}] \sqrt{1-v^2} \mathbb E [M_n] \Big\}.
\end{equation*}

Next, we provide a bound on $E|M_n -\mathbb E [M_n]|$. We use an argument similar to that in the proof of Lemma \ref{lem:prelim_results}, to argue that the distribution of $M_n$ is independent of message $\mathbf{m}$. Recall that, $\mathbf{Y}$ and ${\mathbf{Y}_i}$'s are independent samples from the uniform distribution over $\mathcal{S}(\mathbf{m}) = \{ \mathbf{y} \in \mathcal{S}^{d-1} : \langle \mathbf{y}, \mathbf{m} \rangle = 0 \}$. Due to the rotational symmetry of $\mathcal{S}^{d-1}$, the distribution of $X_i = \langle \mathbf{Y} , \mathbf{Y}_i \rangle$ is independent of $\mathbf{m}$. Therefore, without loss of generality, we can choose $\mathbf{m} = \mathbf{e}_d = (0, 0, \dots, 0, 1)$. For this choice of $\mathbf{m}$, we can generate samples from $\text{Unif}\big(\mathcal{S}(\mathbf{e}_d)\big)$ as follows: first sample $\tilde{\mathbf{Y}} \sim \text{Unif}\big(\mathcal{S}^{d-2}\big)$, then set $\mathbf{Y} = (\tilde{\mathbf{Y}}, 0)$. Similarly, we can write $\mathbf{Y}_i = (\tilde{\mathbf{Y}}_i, 0)$, where $\tilde{\mathbf{Y}}_i$ are independent samples from $\text{Unif}\big(\mathcal{S}^{d-2}\big)$. Consequently, $X_i = \langle \mathbf{Y}, \mathbf{Y}_i \rangle = \langle \tilde{\mathbf{Y}}, \tilde{\mathbf{Y}}_i \rangle$. This implies that $M_n$ follows the same distribution as $ \max_i \langle \tilde{\mathbf{Y}}, \tilde{\mathbf{Y}}_i \rangle$, where $\tilde{\mathbf{Y}}$ and $\tilde{\mathbf{Y}}_i$'s are independent samples from $\text{Unif}\big(\mathcal{S}^{d-2}\big)$.

As such, without loss of generality, we  define $M_n$ as a function of $\tilde{\mathbf{Y}}$ and $\tilde{\mathbf{Y}}_i$'s. Meaning, we  write $M_n(\tilde{\mathbf{Y}}, \tilde{\mathbf{Y}}_1, \dots, \tilde{\mathbf{Y}}_n) := \max_{i \in [n]} \langle \tilde{\mathbf{Y}}, \tilde{\mathbf{Y}}_i \rangle$, where $M_n(\tilde{\mathbf{Y}}, \tilde{\mathbf{Y}}_1, \dots, \tilde{\mathbf{Y}}_n)$ is $1$-Lipschitz in each of its arguments separately.

As presented in \citep{ledoux2001concentration, bakry2013analysis}, the uniform distribution over $\mathcal{S}^{d-2}$ satisfies a Log-Sobolev Inequality (LSI) with constant $1/(d-2)$. By the tensorization property of the LSI \cite[Chapter 5]{ledoux2001concentration}, the product measure on $\mathcal{S}^{d-2} \times (\mathcal{S}^{d-2})^n$ also satisfies an LSI with the same constant $1/(d-2)$. Consequently, for any function $g: \mathcal{S}^{d-2} \times (\mathcal{S}^{d-2})^n \rightarrow \mathbb R$ that is $1$-Lipschitz separately in each of its arguments, we have
\begin{align*}
    \mathbb P \big(|g(\tilde{\mathbf{Y}}) - \mathbb E[g(\tilde{\mathbf{Y}})]| > \varepsilon\big) \leq 2 \exp\left( -\frac{(d-2)\varepsilon^2}{2} \right).
\end{align*}
Since $M_n(\tilde{\mathbf{Y}}, \tilde{\mathbf{Y}}_1, \dots, \tilde{\mathbf{Y}}_n) = \max_{i \in [n]} \langle \tilde{\mathbf{Y}}, \tilde{\mathbf{Y}}_i \rangle$ is $1$-Lipschitz in each of its arguments separately, we obtain
\begin{align*}
\mathbb P \big(|M_n - \mathbb{E}[M_n]| > \varepsilon\big) \leq 2 \exp\left( -\frac{(d-2)\varepsilon^2}{2} \right).
\end{align*}
By integrating this tail, the expected deviation is bounded as
\begin{align*}
\mathbb{E}\big[|M_n - \mathbb{E}[M_n]|\big] \leq \sqrt{\frac{\pi}{2(d-2)}} \leq \sqrt{\frac{\pi}{2d^\star}}.
\end{align*}
Combining this with the arguments in Eq. \eqref{eq:replace_w_expec} and \eqref{eq:replace_Mn_expec}, and using triangle's inequality, we get
\begin{align*}
   \Big| U_n(q) - \mathbb E \Big[ \max_{i \in [n]} \big\{ \mathbb{E}[W] V_i + \mathbb{E}[\sqrt{1-W^2}] \sqrt{1-V_i^2} \mathbb E [M_n] \big\} \Big] \Big| \leq \frac{1}{\sqrt{d^\star}} \Big( \sqrt{\frac{\pi}{2}} + \frac{2\rho \bar K}{(1-\rho^2)} \Big).
\end{align*}
Now, the result follows by taking $K_{\ref{thm:tilt_opt}} = \sqrt{\frac{\pi}{2}} + 2\bar K$. This completes the proof.
\end{proof}

\subsection{Proof of Theorem \ref{thm:tilt}}
\label{app:proof_thm_tilt}
\begin{proof}
To find the expected value of $M_n$, we first consider the special case ($\rho=0$) in Proposition \ref{prop: utility_approx}. From the utility approximation result in Proposition \ref{prop: utility_approx}, given a message $\mathbf{m}$ and a recommendation set size $n = \lfloor e^{d \alpha} \rfloor$, the expected utility for $\rho=0$ is
\begin{align*}
\lim_{d\rightarrow \infty}\mathbb{E} \Big[ \max_{i \in [n]} \langle \mathbf{h}, \boldsymbol{\theta}_i \rangle \Big] = f(0, \alpha),
\end{align*}
where $f(0, \alpha)$ is the asymptotic utility defined by the large deviations rate function $I_0(w, x)$:
\begin{align*}
f(0, \alpha) := \max_{w, x \in (-1, 1)} \left\{ \sqrt{1-w^2} x \right\} \quad \text{subject to} \quad I_0(w, x) \leq \alpha,
\end{align*}
where $I_0(w, x) = -\frac{1}{2} \log(1-w^2) - \frac{1}{2} \log(1-x^2)$. The maximum is attained at $w=0$ and $x = \sqrt{1 - e^{-2\alpha}}$, giving us
\begin{align*}
f(0, \alpha) = \sqrt{1 - e^{-2\alpha}}.
\end{align*}
This result characterizes the expected maximum of inner products for vectors drawn uniformly from $\mathcal{S}^{d-1}$. Specifically, for $\tilde{\mathbf{Y}}, \tilde{\mathbf{Y}}_1, \dots, \tilde{\mathbf{Y}}_n$ i.i.d. and uniform over $\mathcal{S}^{d-1}$
\begin{align*}
\lim_{d\rightarrow \infty} \mathbb{E} \left[ \max_{i \in [n]} \langle \tilde{\mathbf{Y}}, \tilde{\mathbf{Y}}_i \rangle \right] = \sqrt{1 - e^{-2\alpha}}.
\end{align*}
Finally, we apply this to the maximum of $X_i = \langle \mathbf{Y}, \mathbf{Y}_i \rangle$. Although $\mathbf{Y}$ and $\mathbf{Y}_i$ are uniform over the $(d-2)$-dimensional hypersphere $\mathcal{S}^{d-2}$ rather than $\mathcal{S}^{d-1}$, the asymptotic behavior of the rate function remains identical for large $d$. Thus, the expected maximum of the search components is:
\begin{align*}
\lim_{d\rightarrow \infty} \mathbb{E}[M_n] = \lim_{d\rightarrow \infty} \mathbb{E} \left[ \max_{i \in [n]} X_i \right] = \sqrt{1 - e^{-2\alpha}}.
\end{align*}
From Lemma \ref{lem:w_concentration} (and as argued in the proof of Proposition \ref{prop: utility_approx}), we have
\begin{align*}
\lim_{d\rightarrow \infty} \mathbb{E}[W] = \rho \quad \text{and} \quad \lim_{d\rightarrow \infty} \mathbb{E}[\sqrt{1-W^2}] = \sqrt{1-\rho^2}.
\end{align*}
By setting the concentration parameter $\kappa_d(\rho) = \frac{\rho}{1-\rho^2}(d-3)$ and the set size $n_d(\alpha) = \lfloor e^{\alpha d} \rfloor$ for fixed $\rho$ and $\alpha$, for any choice of tilt parameter $v$, we have
\begin{align*}
\lim_{d \to \infty} \mathcal{T}_d(\kappa_d, n_d, v) =\mathcal{T}_{\infty}(\rho, \alpha, v),
\end{align*}
where the limiting utility is given by:
\begin{align*}
\mathcal{T}_{\infty}(\rho, \alpha, v) = \rho v + \sqrt{1-\rho^2}\sqrt{1-v^2}\sqrt{1-e^{-2\alpha}}.
\end{align*}
To find the optimal tilt parameter $v^*$ corresponding to $\mathcal{T}_{\infty}$, we simply optimize $\mathcal{T}_{\infty}(\rho, \alpha, v)$ for fixed $\rho$ and $\alpha$, which gives,
\begin{align*}
    v^*(\rho, \alpha) = \frac{\rho}{\sqrt{\rho^2 + (1-\rho^2)(1-e^{-2\alpha})}}.
\end{align*}
Substituting $v^*(\rho, \alpha)$ back into $\mathcal{T}_{\infty}(\rho, \alpha, v)$, the maximized utility (at the optimal tilt) simplifies to
\begin{align*}
\max_v \mathcal{T}_{\infty}(\rho, \alpha, v) = \sqrt{\rho^2 + (1-\rho^2)(1-e^{-2\alpha})}.
\end{align*}
Next, we aim to solve for $\rho^*$ and $\alpha^*$, where
\begin{align*}
   (\rho^*, \alpha^*) &= \arg \max_{\rho \in [0,1), \alpha \geq 0} \left\{ \sqrt{\rho^2 + (1-\rho^2)(1-e^{-2\alpha})} - c_s \alpha + \frac{1}{2} c_c \log(1-\rho^2) \right\}\\
   &= \arg \max_{\rho \in [0,1), \alpha \geq 0} \left\{ \sqrt{1-(1-\rho^2)e^{-2\alpha}} + \frac{1}{2}  c_s\log e^{-2\alpha} + \frac{1}{2} c_c \log(1-\rho^2) \right\}.
\end{align*}
As can be easily observed: (i) when $c_c > c_s$, it is optimal to set $\rho=0$, i.e., $\rho^* = 0$ which also gives $v^* = 0$, as $v^*(0,\alpha)=0$ for any $\alpha > 0$, (ii) Conversely, when $c_c < c_s$, it is optimal to set $\alpha=0$, i.e., $\alpha^* = 0$ which also gives $v^* = 1$, as $v^*(\rho,0)=1$ for any $\rho > 0$.

Further, by representing $c_{\text{min}} = \min\{c_s, c_c\}$ and suppose $z = (1-\rho^2)e^{-2\alpha}$, then we have
\begin{align}
\label{eq:c_min_bound}
    \max_{\rho \in [0,1), \alpha \geq 0} & \left\{ \sqrt{1-(1-\rho^2)e^{-2\alpha}} + \frac{1}{2}  c_s\log e^{-2\alpha} + \frac{1}{2} c_c \log(1-\rho^2) \right\} \nonumber\\
    & \leq \max_{\rho \in [0,1), \alpha \geq 0}  \left\{ \sqrt{1-(1-\rho^2)e^{-2\alpha}} + \frac{1}{2}  c_{\text{min}}\log e^{-2\alpha} + \frac{1}{2} c_{\text{min}} \log(1-\rho^2) \right\} \nonumber\\
    & = \max_{z \in [0,1)} \left\{ \sqrt{1-z} + \frac{1}{2}  c_{\text{min}}\log z \right\}.
\end{align}
By solving the above optimization problem, we get
\begin{align*}
    z^* = \frac{1}{2}  \Big( \sqrt{c_{\min}^4+ 4 c_{\min}^2} -c_{\min}^2 \Big).
\end{align*}
and the equality in Eq. \eqref{eq:c_min_bound} holds when either $\rho^* =0$ (if $c_c > c_s$) or $\alpha^* =0$ (if $c_c < c_s$), and
\begin{align*}
    \big(1-(\rho^*)^2\big)e^{-2\alpha^*} = z^* = \frac{1}{2}  \Big( \sqrt{c_{\min}^4+ 4 c_{\min}^2} -c_{\min}^2 \Big).
\end{align*}
This completes the proof.
\end{proof}

\end{document}